%% file: main.tex
\documentclass{article}
\usepackage[tracking=true]{microtype}
\usepackage{graphicx}
\usepackage{subfigure}
\usepackage{booktabs} 
\usepackage{amsmath}
\usepackage{amsthm}
\usepackage[lighttt]{lmodern}
\usepackage{amssymb}
\usepackage[ruled,vlined]{algorithm2e}

\newtheorem{lemma}{Lemma}

\usepackage{listings}

\newcommand*{\textcompress}[1]{\textls[-10]{#1}}

\usepackage{hyperref}

\usepackage[accepted]{icml2021}
\usepackage[T1]{fontenc}    
\usepackage{hyperref}       
\usepackage{url}            
\usepackage{booktabs}       
\usepackage{amsfonts}       
\usepackage{nicefrac}       

\usepackage{xifthen}
\usepackage{graphicx}
\usepackage{subfigure}
\usepackage{lipsum}
\usepackage{mathtools}
\usepackage{cuted}
\usepackage{times}
\usepackage{url}
\usepackage[makeroom]{cancel}
\usepackage{latexsym}
\usepackage{pifont}
\usepackage{amsmath}
\usepackage{footnote}
\usepackage{graphicx}
\usepackage{amsfonts}
\usepackage{bm}
\usepackage[mode=text]{siunitx}
\usepackage{thmtools, thm-restate}
\usepackage{amsmath, amsthm, amssymb}
\usepackage{etoolbox}
\usepackage{multirow}
\usepackage{caption}
\usepackage{amsmath,amsfonts,amssymb}
\usepackage{mathrsfs}
\usepackage{tabu}
\usepackage{adjustbox}
\usepackage{color}
\usepackage{ulem}  

\definecolor{emerald}{rgb}{0.31, 0.78, 0.47}
\definecolor{lightcoral}{rgb}{0.94, 0.5, 0.5}
\definecolor{gold(web)(golden)}{rgb}{1.0, 0.84, 0.0}
\definecolor{lightcornflowerblue}{rgb}{0.6, 0.81, 0.93}
\usepackage{hyperref}

\newcommand*{\mi}{\ensuremath{I}}

\newcommand{\eqdef}{\triangleq}

\newcommand{\p}{x}
\newcommand{\f}{y}

\newcommand*{\uprime}{^{\prime}\mkern-1.2mu}

\newcommand{\expect}{\mathbb{E}}
\newcommand{\cri}{\psi}
\newcommand{\criopt}{\psi^*}
\newcommand{\crip}{\phi}

\newcommand\INCE{\ensuremath{\textrm{InfoNCE}}}

\newcommand{\x}{\ensuremath{x}}
\newcommand{\xt}{y}
\newcommand{\xp}{x}
\newcommand{\ps}{x}
\newcommand{\xpp}{x\uprime}

\newcommand{\xup}{x\uprime}

\icmltitlerunning{}

\begin{document}

\twocolumn[
\icmltitle{Decomposed Mutual Information Estimation \\ for Contrastive Representation Learning}

\begin{icmlauthorlist}
\icmlauthor{Alessandro Sordoni*}{m}
\icmlauthor{Nouha Dziri*}{u}
\icmlauthor{Hannes Schulz*}{m}
\icmlauthor{Geoff Gordon}{m}
\icmlauthor{Phil Bachman}{m}
\icmlauthor{Remi Tachet}{m}
\end{icmlauthorlist}

\icmlaffiliation{m}{Microsoft Research}
\icmlaffiliation{u}{University of Alberta}

\icmlcorrespondingauthor{Alessandro Sordoni}{alsordon@microsoft.com}
\icmlcorrespondingauthor{Nouha Dziri}{dziri@cs.ualberta.ca}

\icmlkeywords{}

\vskip 0.3in
]

\printAffiliationsAndNotice{\icmlEqualContribution} 

\begin{abstract}
Recent contrastive representation learning methods rely on estimating mutual information (MI) between multiple views of an underlying context. E.g., we can derive multiple views of a given image by applying data augmentation, or we can split a sequence into views comprising the past and future of some step in the sequence. Contrastive lower bounds on MI are easy to optimize, but have a strong underestimation bias when estimating large amounts of MI. We propose decomposing the full MI estimation problem into a sum of smaller estimation problems by splitting one of the views into progressively more informed subviews and by applying the chain rule on MI between the decomposed views. This expression contains a sum of unconditional and conditional MI terms, each measuring modest chunks of the total MI, which facilitates approximation via contrastive bounds. To maximize the sum, we formulate a contrastive lower bound on the conditional MI which can be approximated efficiently. We refer to our general approach as Decomposed Estimation of Mutual Information (DEMI). We show that DEMI can capture a larger amount of MI than standard non-decomposed contrastive bounds in a synthetic setting, and learns better representations in a vision domain and for dialogue generation.

\end{abstract}

\input{SEC_introduction}

\input{SEC_infonce}
\input{SEC_infobounds}
\input{SEC_experiments}

\input{SEC_related}
\input{SEC_conclusion}

\bibliography{icml2021-normalized}
\bibliographystyle{icml2021}

\appendix
\input{SEC_appendix}
\end{document}

%% file: SEC_introduction.tex
\section{Introduction}

The ability to extract actionable information from data in the absence of explicit supervision seems to be a core prerequisite for building systems that can, for instance, learn from few data points or quickly make analogies and transfer to other tasks. Approaches to this problem include generative models~\citep{hinton2012practical,vae} and self-supervised representation learning approaches, in which the objective is not to maximize likelihood, but to formulate a series of (label-agnostic) tasks that the model needs to solve through its representations~\citep{noroozi2016unsupervised,devlin2018bert,gidaris2018unsupervised,hjelm2018learning}. Self-supervised learning includes successful models leveraging contrastive learning, which have recently attained comparable performance to their fully-supervised counterparts~\citep{amdim,chen2020simple}.

Recent self-supervised learning methods can be seen as training an encoder $f$ such that it maximizes the mutual information (MI) between representations $f(\cdot)$ of a pair of views $x$ and $y$ of the same input datum, $\mi(f(\x); f(\xt)) \le \mi(\x; \xt)$\footnote{In what follows, we will slightly abuse language and use the expression ``maximizing $\mi(\x, \xt)$'' as a shortcut for ``maximizing a lower bound on $\mi(\x, \xt)$ with respect to $f$''.}. For images, different views can be built using random flipping or color jittering~\citep{amdim,chen2020simple}. For sequential data such as conversational text, the views can be past and future utterances in a given dialogue, or a particular word and its surrounding context~\citep{stratos2018mutual}.
Contrastive approaches train representations of pairs of views to be more similar to each other than to representations sampled from a negative sample distribution. The \INCE{} bound on $\mi(\x; \xt)$ ~\citep{oord2018representation} has been successful insofar as it enjoys much lower variance than competing approaches~\citep{song2019understanding}. However, the capacity of the bound is limited by the number of contrastive samples used~\citep{mcallester2018formal,poole2019variational} and is therefore likely biased when a large amount of MI needs to be estimated,~e.g. between high dimensional objects such as natural images.

\begin{figure*}[th]
    \centering
    \includegraphics[scale=0.40]{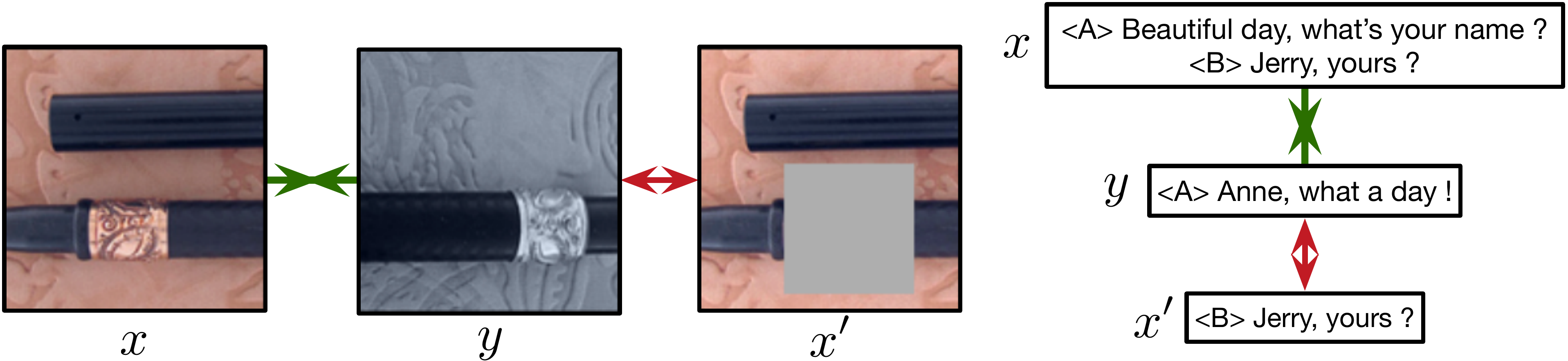}
    \caption{(\textbf{left}) Given two augmentations $\x$ and $\xt$, we create a subview $\xup$, which is obtained by occluding some of the pixels in $\xp$. We can maximize $\mi(\x; \xt) \ge \mi(\xup; \xt) + \mi(\xp; \xt | \xup)$ using a contrastive bound by training $\xup$ to be closer to $\xt$ than to other images from the corpus. Additionally, we train $\xp$ to be closer to $\xt$ than to samples from $p(\xt | \xup)$, i.e.~we can use $\xup$ to generate hard negatives $\xt$, which corresponds to maximizing conditional MI, and leads the encoder to capture features not explained by $\xup$. (\textbf{right}) A fictional dialogue in which $\xp$ and $\xt$ represent past and future of the conversation respectively and $\xup$ is the ``recent past''. In this context, the conditional MI term encourages the encoder to capture long-term dependencies that cannot be explained by the most recent utterances.}
    \label{fig:figure_intro}
\end{figure*}

The starting point of this paper is to decompose $\mi(\x, \xt)$ by applying the chain rule on MI to obtain a sum of terms, each containing smaller chunks of the total MI that can be approximated with less bias by contrastive approaches. For example, consider creating a subview $\xup$ by removing information from $\x$, e.g.~by masking some pixels as depicted in Fig.~\ref{fig:figure_intro} (\textit{left}). By construction, $\mi(\xup, \xp; \xt) = \mi(\xup; \xt) + \mi(\xp; \xt | \xup) = \mi(\x; \xt)$. Decomposed Estimation of Mutual Information (DEMI) prescribes learning representations that maximize each term in the sum, by contrastive learning. The conditional MI term measures the information about $\xt$ that the model has gained by looking at $\xp$ given the information already contained in $\xup$. An intuitive explanation of why this term may lead to capturing more of the total MI between views can be found in Fig.\,\ref{fig:figure_intro}. For images (\textit{left}), only maximizing $\mi(\xp; \xt)$ could imbue the representations with the overall ``shape'' of the stick and representations would likely need many negative samples to capture other discriminative features of the image. By maximizing conditional MI, we hope to more directly encourage the model to capture these additional features, e.g.~the embossed detailing. In the context of predictive coding on sequential data such as dialogue, by setting $\xup$ to be the most recent utterance (Fig.\,\ref{fig:figure_intro}, \textit{right}), the encoder is directly encouraged to capture long-term dependencies that cannot be explained by the most recent utterance.

\textls[-5]{One may wonder how DEMI is related to recent approaches maximizing MI between more than two views, amongst them AMDIM~\citep{amdim}, CMC~\citep{tian2019contrastive} and SwAV~\citep{caron2020unsupervised}.} Interestingly, these models can be seen as maximizing the sum of MIs between views $\mi(\x, \xup; \xt) = \mi(\xup; \xt) + \mi(\xp; \xt)$. E.g.,~in~\citet{amdim}, $\xp$ and $\xup$ could be global and local representations of an image, and in~\citet{caron2020unsupervised}, $\xp$ and $\xup$ could be the views resulting from standard cropping and the aggressive multi-crop strategy. This equality is only valid when the views $\xp$ and $\xup$ are statistically independent, which usually does not hold. Instead, DEMI maximizes $\mi(\xp, \xup; \xt) = \mi(\xup; \xt) + \mi(\xp; \xt | \xup)$, which always holds. Most importantly, the conditional MI term encourages the encoder to capture more non-redundant information across views.

Our contributions are the following. We show that DEMI can potentially capture more of the total information shared between the original views $\x$ and $\xt$. We extend existing contrastive MI bounds to conditional MI estimation and present novel computationally tractable approximations. Supplementally, our results offer another perspective on \textit{hard} contrastive examples,~i.e.,~\citet{faghri2017vse++}, given that conditional MI maximization can be achieved by sampling contrastive examples from a partially informed conditional distribution instead of the marginal distribution. We first show in a synthetic setting that DEMI leads to capturing more of the ground-truth MI thus alleviating bias existing in InfoNCE. Finally, we present evidence of the effectiveness of the proposed method in vision and in dialogue generation.

%% file: SEC_infonce.tex
\section{Problem Setting}

The maximum MI predictive coding framework~\citep{mcallester2018information,oord2018representation,hjelm2018learning} prescribes learning representations of input data such that they maximize MI between inputs and representations. Recent interpretations of this principle create two independently-augmented copies $\x$ and $\xt$ of the same input by applying a set of stochastic transformations twice, and then learn representations of $\x$ and $\xt$ by maximizing the MI of the respective features produced by an encoder $f: \mathcal{X} \rightarrow \mathbb{R}^d$~\citep{amdim,chen2020simple}:
\begin{align}
\arg\max_{f}\; \mi(f(\x); f(\xt)) \le \mi(\x; \xt)
\label{eq:gen_mi}
\end{align}
where the upper bound is due to the data processing inequality. Our starting point to maximize Eq.\,\ref{eq:gen_mi} is the recently proposed \INCE{} lower bound on MI~\citep{oord2018representation} which trains $f(\x)$ to be closer to $f(\xt)$ than to the representations of other images drawn from the marginal distribution of the corpus. This can be viewed as a \textit{contrastive} estimation of the MI~\citep{oord2018representation} and has been shown to enjoy lower variance than competing approaches~\citep{song2019understanding}.

\subsection{InfoNCE Bound}


\INCE{}~\citep{oord2018representation} is a lower-bound on $\mi(\x; \xt)$ obtained by comparing pairs sampled from the joint distribution $x, y_1 \sim p(\x, \xt)$ to pairs $x, y_i$ built using a set of negative examples, $y_{2:K} \sim p(y_{2:K}) = \prod_{k = 2}^K p(y_k)$, also called \textit{contrastive}, independently sampled from the marginal:
\begin{align}
I_{\textit{NCE}}(\x, \xt | \crip, K) = \mathbb{E} \left[\log \frac{e^{\cri(\x, \xt_1)}}{\frac{1}{K} \sum_{k=1}^K e^{\cri(\x, \xt_k)}}\right],
\label{eq:infonce}
\end{align}
where the expectation is with respect to $p(\x, \xt_1)p(\xt_{2:K})$ and $\cri$ is a critic assigning a real valued score to $x, y$ pairs. Usually, $\cri$ is the dot product of the representations after applying an additional transformation $g$, e.g.~an MLP, $\cri(x, y) \eqdef g(f(x))^T g(f(y))$~\citep{chen2020simple}. We provide an exact derivation of this bound in the Appendix\footnote{The derivation in~\citet{oord2018representation} presented an approximation and therefore was not properly a bound. An alternative, exact derivation of the bound can be found in~\cite{poole2019variational}.}.
The optimal value of $I_{\textit{NCE}}$ is reached for a critic proportional to the log-odds between the conditional distribution $p(\xt | \x)$ and the marginal distribution $p(\xt)$, i.e. the PMI between $\x$ and $\xt$, $\criopt(\x, \xt) = \log \frac{p(\xt | \x)}{p(\xt)} + c(\x)$~\citep{oord2018representation,ma2018noise,poole2019variational}.

\INCE{} has recently been extensively used in self-supervised representation learning given that it enjoys lower variance than some of its competitors such as MINE~\citep{pmlr-v80-belghazi18a,song2019understanding}.
However, the bound is loose if the true mutual information $\mi(\x; \xt)$ is larger than $\log K$, which is likely when dealing with high-dimensional inputs such as natural images. To overcome this difficulty, recent methods either train with large batch sizes~\citep{chen2020simple} or exploit an external memory of negative samples in order to reduce memory requirements~\citep{chen2020improved,tian2020makes}. These methods rely on uniform sampling from the training set in order to form the contrastive sets. Discussion of limits of variational bounds can be found in~\citet{mcallester2018formal}.

%% file: SEC_infobounds.tex

\section{Decomposing Mutual Information}
When $\mathcal{X}$ is high-dimensional, the amount of mutual information between $x$ and $y$ will potentially be larger than the amount of MI that $I_{\textit{NCE}}$ can measure given computational constraints associated with large $K$ and the poor log scaling properties of the bound. We argue that we can ease this estimation problem by creating subviews of $x$ and applying the chain rule on MI to decompose the total MI into a sum of potentially smaller MI terms.

By the data processing inequality, we have: $\mi(\x; \xt) \ge \mi(\{\x^1, \ldots, \x^N\}; \xt)$, where $\{\x^1, \ldots, \x^N\}$ are different subviews of $\x$ -- i.e., views derived from $\x$ without adding \textit{any} exogenous information. For example, $\{\x^1, \ldots, \x^N\}$ can represent single utterances in a dialog $\x$, sentences in a document $\x$, or different augmentations of the same image $\x$. Equality is obtained when the set of subviews retains all information about $\x$ or if $\x$ is in the set.

For ease of exposition and without loss of generality, we consider the case where we have two subviews, $x$ itself and $\xup$. Then, $\mi(\x; \xt) = \mi(\x, \xup; \xt)$ and we can write $\mi(\x, \xup; \xt)$ by applying the chain rule for MI:
\begin{equation}
\mi(\x, \xup; \xt) = \mi(\xup; \xt) + \mi(\x; \xt | \xup).
\label{eq:mi-chain-rule}
\end{equation}

The conditional MI term can be written as:
\begin{equation}
\mi(\xp; \xt | \xup) = \mathbb{E}_{p(\xp, \xup, \xt)} \log \frac{p(\xt | \x, \xup)}{p(\xt | \xup)}.
\label{eq:condmi}
\end{equation}
This conditional MI is different from the unconditional MI, $\mi(\xp; \xt)$, as it measures the amount of information shared between $\xp$ and $\xt$ that cannot be explained by $\xup$.

Lower bounding each term in Eq.~\ref{eq:mi-chain-rule} with a contrastive bound can potentially lead to a less biased estimator of the total MI. This motivates us to introduce DEMI, a sum of unconditional and conditional lower bounds:
\begin{equation}
I_{\textit{DEMI}} = I_{\textit{NCE}}(\xup; \xt) + I_{\textit{CNCE}}(\x; \xt | \xup) \leq \mi(\x; \xt),
\label{eq:sep_info_nce}
\end{equation}
where $I_{\textit{CNCE}}$ is a placeholder for a lower bound on the conditional MI and will be presented in the next section. Both conditional and unconditional bounds on the MI can capture at most $\log K$ nats of MI. Therefore, DEMI in Eq.~\ref{eq:sep_info_nce} potentially allows to capture up to $N \log K$ nats of MI in total, where $N$ is the number of subviews used to describe $\x$. This is strictly larger than $\log K$ in the standard $I_{\textit{NCE}}$.

\section{Contrastive Conditional MI Estimation}
One of the difficulties in computing DEMI is estimating the conditional MI. In this section, we provide bounds and approximations of this quantity. First, we show that we can readily extend InfoNCE:

\begin{restatable}[\textbf{Conditional InfoNCE}]{prop}{cnce} $I_{\textit{CNCE}}$ is a lower-bound on $\mi(\xp; \xt | \xpp)$ and verifies the properties below:
\begin{align}
I_{\textit{CNCE}}(\xp; \xt | \xpp, \crip, K) = \mathbb{E} \bigg[ \log \frac{e^{\crip(\xpp, \xp, \xt_1)}}{\frac{1}{K}\sum_{k=1}^K e^{\crip(\xpp, \xp, \xt_k)}}\bigg],
\label{eq:condmibound}
\end{align}
\vspace*{-9mm}
\begin{enumerate}
\item $I_{\textit{CNCE}} \le \mi(\xp; \xt | \xpp)$.
\item $\crip^* = \arg \sup_\crip I_{\textit{CNCE}} = \log \frac{p(\xt | \xpp, \xp)}{p(\xt | \xpp)} + c(\xp, \xpp)$.
\item $\lim_{K \rightarrow \infty} I_{\textit{CNCE}}(\xp; \xt | \xpp, \crip^*, K) = \mi(\xp; \xt | \xpp)$.

\end{enumerate}
\end{restatable}
\vspace*{-1mm}

The expectation is taken with respect to $p(\ps, \xpp, \xt_1)p(\xt_{2:K} | \xpp)$ and the expression is upper bounded by $\log K$. The proof can be found in Sec.~\ref{app:proof_condmi} and follows closely the derivation of the \INCE{} bound by applying a result from~\cite{barber2003algorithm}. A related derivation of this bound was also presented in~\citet{pmlr-v108-foster20a} for optimal experiment design.

Eq.~\ref{eq:condmibound} shows that a lower bound on the conditional MI can be obtained by sampling contrastive sets from the proposal distribution $p(\xt | \xpp)$ (instead of from the marginal $p(\xt)$ as in Eq.~\ref{eq:infonce}). Indeed, since we want to estimate the MI conditioned on $\xpp$, we should allow our contrastive distribution to condition on $\xpp$. Note that $\crip$ is now a function of three variables. One of the biggest hurdles in computing Eq.~\ref{eq:condmibound} is the access to many samples from $p(\xt | \xpp)$, which is unknown and usually challenging to obtain. In order to overcome this, we propose various solutions next.

\subsection{Variational Approximation}

It is possible to obtain a bound on the conditional MI by approximating the unknown conditional distribution $p(\xt | \xpp)$ with a variational distribution $q_{\xi}(\xt | \xpp)$, leading to the following proposition:

\begin{restatable}[\textbf{Variational $I_{\textit{CNCE}}$}]{prop}{ivar} For any variational approximation $q_\xi(\xt | \xpp)$ in lieu of $p(\xt | \xpp)$,
with $p(\cdot | \xpp) \ll q_\xi(\cdot | \xpp)$ for any $\xpp$, we have:
\begin{align}\label{eq:variational-approx}
&I_{\textit{VAR}}(\xp, \xt | \xpp, \phi, \xi, K) = \\ &\; \mathbb{E} \bigg[\log \frac{e^{\crip(\xpp, \xp, \xt_1)}}{\frac{1}{K}\sum_{k=1}^K e^{\crip(\xpp, \xp, \xt_k)}} \bigg] \notag - \mathbb{E} \bigg[ KL \left( p(y|\xpp) \,\|\, q_\xi\right)\bigg],
\end{align}
\begin{enumerate}
\item $I_{\textit{VAR}} \le \mi(\xp; \xt | \xpp)$.
\item If $q_\xi(y | \xpp) = p(y | \xpp)$, $I_{\textit{VAR}} = I_{\textit{CNCE}}$.
\item $\lim_{K \rightarrow \infty} \sup_{\crip} I_{\textit{VAR}}(\xp; \xt | \xpp, \crip, \xi, K) = \mi(\xp; \xt | \xpp)$.
\end{enumerate}
\label{prop:ivar}
\end{restatable}

where the first expectation is taken with respect to $p(\xp, \xpp, 
\xt_1)q_{\xi}(\xt_{2:K} | \xpp)$ and the second with respect to $p(\xpp)$. See Sec.~\ref{app:proof_variational} for the proof. Note that this bound side-steps the problem of requiring access to an arbitrary number of negative samples from the unknown $p(\xt | \xpp)$ by i.i.d. sampling from the known and tractable $q_\xi(\xt | \xpp)$. For example, $q_\xi$ can be a conditional flow-based image generation model~\citep{kingma2018glow} or a transformer language model for text~\citep{zhang2019dialogpt}. We prove that as the number of examples goes to $\infty$, optimizing the bound w.r.t. $\crip$ converges to the true conditional MI. Interestingly, this holds true for any $q_\xi$, though the choice of $q_\xi$ will most likely impact the convergence rate of the estimator.

Eq.~\ref{eq:variational-approx} is superficially similar to the ELBO (Evidence Lower BOund) objective used to train VAEs~\citep{vae}, where $q_\xi$ plays the role of the approximate posterior (although the KL direction in the ELBO is inverted). This parallel suggests that, assuming the variational family contains $p$, the optimal solution w.r.t. $\xi$ may not verify $p(\xt | \xpp) = q_\xi(\xt | \xpp)$ for all values of $K$ and $\phi$,~i.e. there could be solutions for which some of the KL divergence is traded for additional nats on the contrastive cost. However, we see trivially that if we ignore the dependency of the first expectation term on $q_\xi$ (i.e. we ``detach'' the gradient of the expectation w.r.t $\xi$) and only optimize $\xi$ to minimize the KL term, then it is guaranteed that $p(\xt | \xpp) = q_\xi(\xt | \xpp)$, for any $K$ and $\phi$. Thus, by the second property in Proposition~\ref{prop:ivar}, optimizing $I_{\text{\textit{VAR}}}(\phi, \xi^*, K)$ w.r.t. $\phi$ will correspond to optimizing $I_{\text{\textit{CNCE}}}$.

In practice, the latter observation significantly simplifies the estimation problem as one can minimize a Monte-Carlo approximation of the KL divergence w.r.t $\xi$ by standard supervised learning: we can efficiently approximate the KL by taking samples from $p(\xt | \xpp)$. Those can be directly obtained by using the joint samples from $p(x, \xt)$ included in the training set and computing $\xpp$ from $\x$.\footnote{The ability to perform that computation is usually a key assumption in self-supervised learning approaches.} However, maximizing $I_{\textit{VAR}}$ can still be challenging as it requires estimating a distribution over potentially high-dimensional inputs and efficiently sampling a large number of negative examples from it. In the next section, we provide an importance sampling approximation of $I_{\textit{CNCE}}$ that bypasses this issue.

\subsection{Importance Sampling Approximation}

The optimal critic for $I_{\textit{NCE}}$ is $\cri^*(\xpp, \xt) = \log \frac{p(\xt | \xpp)}{p(\xt)} + c(\xpp)$, for any $c$. Assuming access to $\cri^*(\xpp, \xt)$, it is possible to use importance sampling to produce approximate expectations from $p(\xt | \xpp)$. This is achieved by first sampling $\tilde\xt_{1:M} \sim p(\xt)$ and then resampling $K \le M$ ($K > 0$) examples i.i.d. from the normalized importance distribution $w_k = \frac{\exp \cri^*(\xpp, \tilde\xt_k)}{\sum_{{m=1}}^M \exp \cri^*(\xpp, \tilde\xt_m)}$. This process is also called ``sampling importance resampling'' (SIR) and we can write the corresponding distribution as $p_{\textit{SIR}}(y_k) = w_k \delta(\xt_k \in \tilde\xt_{1:M})p(\tilde\xt_{1:M})$. As $M/K \rightarrow \infty$, it is guaranteed to produce samples from $p(\xt | \xpp)$~\citep{rubin1987}.

The objective corresponding to this process is:
\begin{align}
\label{eq:sir}
I_{\textit{SIR}}(&\xp, \xt | \xpp, \crip, K) = \\  & \mathbb{E}_{p(\xpp, \xp, \xt_1)p_{\textit{SIR}}(\xt_{2:K})} \bigg[ \log \frac{e^{\crip(\xpp, \xp, \xt_1)}}{\frac{1}{K} \sum_{k=1}^K e^{\crip(\xpp, \xp, \xt_k)}} \bigg] \notag
\end{align}
Note the dependence of $p_{\textit{SIR}}$ on $w_k$ and hence $\cri^*$. SIR is known to increase the variance of the estimator~\citep{skare2003improved} and is wasteful given that only a smaller set of $K < M$ examples are actually used for MI estimation.

To provide a cheap approximation of the SIR estimator, we split the denominator of Eq.~\ref{eq:sir} into a positive term involving $\xt_1$ and a sum of contributions coming from negative examples $\xt_{2:K}$, and we rewrite the latter as an average $(K-1) \sum_{k=2}^K \frac{1}{K-1} e^{\crip(\xpp, \xp, \xt_k)}$. Now, we can use the normalized importance weights $w_k$ to estimate that term under the resampling distribution. Formally, we have the following approximation:

\begin{restatable}[\textbf{Importance Sampled $I_{\textit{CNCE}}$}]{prop}{is} Assuming $\cri^* = \arg\sup_\cri I_{\textit{NCE}}(\xpp, \xt)$ and $w_k = \frac{\exp \cri^*(\xpp, \xt_k)}{\sum_{{k=2}}^M \exp \cri^*(\xpp, \xt_m)}$, we have the following two properties, where:
\begin{align}
&I_{\textit{IS}}(\xp, \xt | \xpp, \crip, K) = \notag \\ &\mathbb{E} \left[ \log \frac{e^{\crip(\xpp, \xp, \xt_1)}}{\frac{1}{K} (e^{\crip(\xpp, \xp, \xt_1)} + (K - 1) \;{\sum_{k=2}^K w_k e^{\crip(\xpp, \xp, \xt_k)}})} \right],
\label{eq:cond_imp_samp}
\end{align}
\begin{enumerate}
\item $\lim_{K \rightarrow \infty} \sup_\crip I_{\textit{IS}}(\xp; \xt | \xpp, \crip, K) = \mi(\xp; \xt | \xpp),$
\item $\lim_{K \rightarrow \infty} \arg\sup_\crip I_{\textit{IS}} = \log \frac{p(y | \xpp, \xp)}{p(y | \xpp)} + c(\xp, \xpp)$.
\end{enumerate}
\label{prop:iis}
\end{restatable}

where the expectation is with respect to $p(\xpp, \xp, \xt_1)p(\xt_{2:K})$. The proof can be found in Sec.~\ref{app:proof_is}. $I_{\textit{IS}}$ skips the resampling step by up-weighting the negative contribution to the normalization term of examples that have large probability under the resampling distribution,~i.e.~that have large $w_k$. As detailed in the appendix, this approximation is cheap to compute given that the negative samples are sampled from the marginal distribution $p(\xt)$ and we avoid the need for the resampling step. We hypothesize that $I_{\textit{IS}}$ has less variance than $I_{\textit{SIR}}$ as it does not require the additional resampling step. The proposition shows that as the number of negative examples goes to infinity, the proposed approximation converges to the true value of the conditional MI, and, in the limit of $K \to \infty$, optimizing $I_{\textit{IS}}$ w.r.t. $\crip$ converges to the conditional MI and the optimal $\crip$ converges to the optimal $I_{\textit{CNCE}}$ solution.

\subsection{Boosted Critic Approximation}
Proposition~\ref{prop:iis} shows that the optimal critic $\crip^*$ estimates the desired log ratio only in the limit of $K \to \infty$. Hereafter, we generalize the results presented in~\citet{ma2018noise} and show that we can accurately estimate the conditional log-ratio with the following proposition.

\begin{restatable}[\textbf{Boosted Critic Estimation}]{prop}{bo} Assuming $\cri^* = \arg\sup_\cri I_{\textit{NCE}}(\xup, \xt)$, the following holds, with:
\begin{align}
I_{\textit{BO}}(\x, \xt | \xup, \phi, K) = \mathbb{E} \bigg[ \log \frac{e^{\cri^*(\xup, \xt_1) + \phi(\xup, \x, \xt_1)}}{\frac{1}{K}\sum_{k=1}^K e^{\cri^*(\xup, \xt_k) + \phi(\xup, \x, \xt_k)}} \bigg],
\label{eq:cond_bo}
\end{align}
\vspace*{-8mm}
\begin{enumerate}
\item $I_{\textit{BO}} \le \mi(\x, \xup; \xt)$,
\item $\crip^* = \arg\sup_\crip I_{\textit{BO}} = \log \frac{p(\xt | \xpp, \xp)}{p(\xt | \xpp)} + c(\xp, \xpp)$.
\end{enumerate}
\label{prop:bo}
\end{restatable}
where the expectation is with respect to $p(\x, \xup, \xt_1)p(\xt_{2:K})$. The proof is straightforward and is in Sec.~\ref{proof:bo}.

We refer to Eq.~\ref{eq:cond_bo} as \textit{boosted critic estimation} due to the fact that optimizing $\phi$ captures residual information not expressed in $\psi^*$. Perhaps surprisingly, $I_{\textit{BO}}$ provides an almost embarrassingly simple way of estimating the desired log-ratio \textit{for any} $K$. It corresponds to estimating an InfoNCE like bound, where negative samples come from the easily-sampled marginal $p(y)$ and the critic is shifted by the optimal critic for $I_{\textit{NCE}}(\xup, \xt)$. However, this comes at the cost of not having a valid approximation of the conditional MI. Indeed, by 1., $I_{\textit{BO}}$ is a lower-bound on the total MI, not on the conditional MI. As we show in the next section, we can get an estimate of the conditional MI by using $I_{\textit{BO}}$ to estimate the conditional critic in an accurate manner and $I_{\textit{IS}}$ to evaluate the conditional MI.

%% file: SEC_experiments.tex
\newcommand\mysubfig[1]{\hspace*{1cm}{#1}}
\begin{figure*}[tbp]
\mysubfig{\adjustbox{trim={.0\width} {.15\height} {.0\width} {.04\height},clip}%
{\includegraphics[width=.87\linewidth]{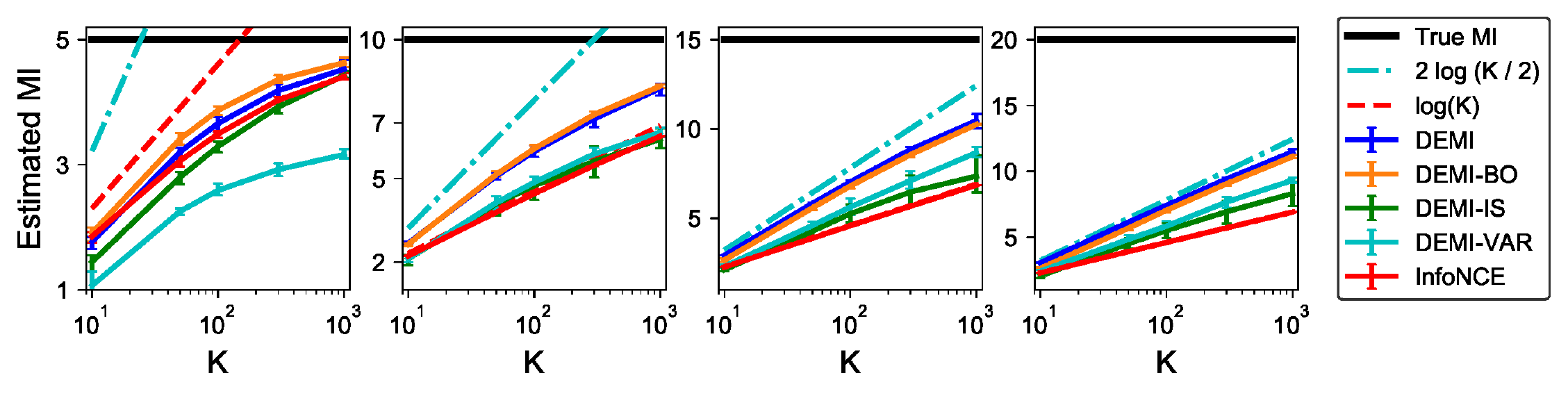}}}
\mysubfig{\adjustbox{trim={.0\width} {.06\height} {.0\width} {.0\height},clip}{\includegraphics[width=.9\linewidth]{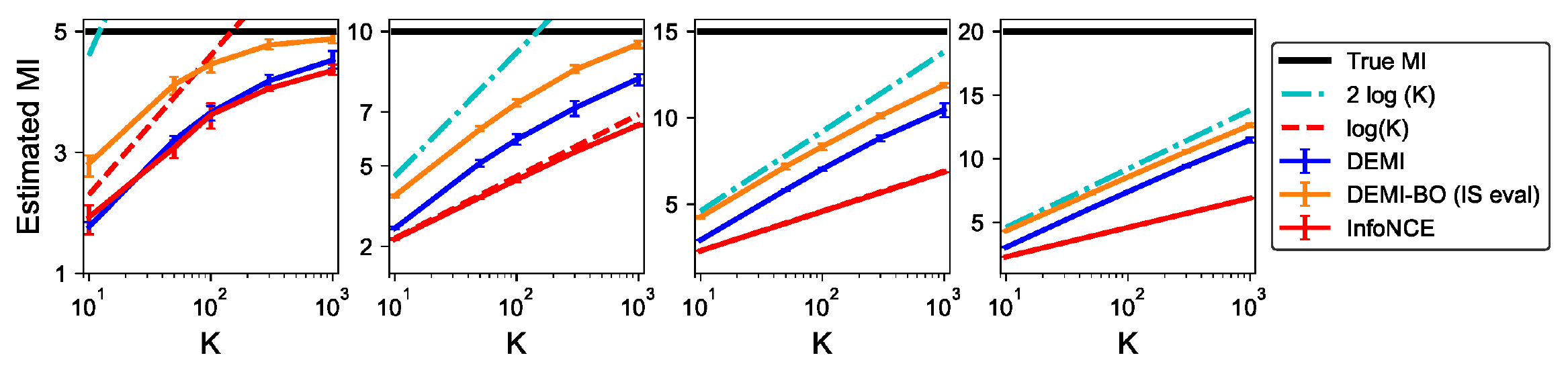}}}
\caption{
    Estimation of $I(\x, \xup; \xt)$ for three Gaussian covariates $\xp, \xpp, \xt$ as function of the number of negative samples $K$. (\textbf{top}) DEMI maximizes $I_{\textit{DEMI}}$ with $K / 2$ examples for unconditional and conditional bounds ($K$ total) and assume access to the ground-truth $p(y | \xp)$. DEMI-IS learns the conditional critic using $I_{\textit{IS}}$, DEMI-BO using $I_{\textit{BO}}$, DEMI-VAR using $I_{\textit{VAR}}$. We plot the total MI estimated by $I_{\textit{DEMI}}$ when learning the conditional critics using our approximations. We see that (1) DEMI captures more MI than InfoNCE for the same number of $K$ and (2) $I_{\textit{BO}}$ accurately estimates the conditional critic without access to samples from $p(y | \xup)$ while $I_{\textit{IS}}$ suffers from significant variance. (\textbf{bottom}) We assess whether we can form a good estimator of the total MI \textit{without} access to $p(y | \xup)$ neither at training nor at evaluation time. Here, DEMI-BO trains the conditional critic by $I_{\textit{BO}}$ and evaluates the total MI by $I_{\textit{NCE}} + I_{\textit{IS}}$.}
    \label{fig:bounds-learned-critic}
    \vspace*{-2.0mm}
\end{figure*}

\vspace{-3mm}
\section{Experiments}
The goal of our experiments is two-fold: (1) to test whether DEMI leads to a better estimator of the total MI, and whether our proposed conditional MI approximations are accurate; (2) to test whether DEMI helps in estimating better representations for natural data. We verify (1) in a synthetic experiment where we control the total amount of MI between Gaussian covariates. Then, we verify (2) on a self-supervised image representation learning domain and explore an additional application to natural language generation in a sequential setting: conversational dialogue.

\subsection{Synthetic Data}
\label{sec:empirical_validation}
We extend~\citet{poole2019variational}'s two variable setup to three variables. We posit that $\{\xp, \xpp, \xt\}$ are three Gaussian co-variates, $\xp, \xpp, \xt \sim \mathcal N(0, \Sigma)$ and we choose $\Sigma$ such that we can control the total mutual information $\mi(\xp, \xpp; \xt)$, $\mi \in \{5, 10, 15, 20\}$ (see Appendix for pseudo-code and details of the setup). We aim to estimate the total MI $\mi(\xp, \xpp; \xt)$ and compare the performance of our approximators in doing so. We limit this investigation to contrastive estimators although other estimators and non lower-bounds exist (e.g. DoE~\citep{pmlr-v108-mcallester20a}). For more details see App.~\ref{app:synthetic}.

In Figure~\ref{fig:bounds-learned-critic} (top), we compare the estimate of the MI obtained by InfoNCE and DEMI, which maximizes $I_{\textit{DEMI}}$ (Eq.~\ref{eq:sep_info_nce}). To be comparable with InfoNCE in terms of total number of negative examples used, DEMI uses half as many negative examples for computing each term in the sum ($K / 2$). For all amounts of true MI, and especially for larger amounts, DEMI can capture more nats than InfoNCE with an order of magnitude less examples. We also report the upper-bounds on InfoNCE ($\log K$) and DEMI ($2 \log K / 2$).

Maximizing $I_{\textit{DEMI}}$ assumes access to negative samples from $p(\xt | \xup)$, which is an unrealistic assumption in practice. To verify the effectiveness of our approximations, we train the conditional critics using $I_{\textit{BO}}$ (DEMI-BO), $I_{\textit{IS}}$ (DEMI-IS) and $I_{\textit{VAR}}$ (DEMI-VAR) and we evaluate the total MI using $I_{\textit{DEMI}}$ (we assume access to $p(y | \xup)$ only at evaluation time). This allows us to verify whether it is possible to reliably estimate the conditional critic in the absence of negative samples from $p(\xt | \xup)$. It is interesting to note how the critic learnt by $I_{\textit{IS}}$ suffers high variance and does not lead to a good estimate of the total MI when evaluated with $I_{\textit{CNCE}}$. DEMI-VAR still outperforms InfoNCE for higher values of total MI, but seems to suffer in the case of small MIs. For this experiment, we update $q_\xi$ at the same rate as $\crip$. Improvements could be obtained by updating $q_\xi$ more frequently, similarly to the asynchronous updates successfully used in the GAN literature~\citep{mescheder2018training}. $I_{\textit{BO}}$ accurately estimates the critic.

In Figure~\ref{fig:bounds-learned-critic} (bottom), we show that it is possible to obtain an estimate of the total MI without access to $p(\xt | \xup)$ neither at training nor evaluation time. We first learn the conditional critic using $I_{\textit{BO}}$ and compute $I_{\textit{NCE}} + I_{\textit{IS}}$ by using the estimated critic. Figure~\ref{fig:bounds-learned-critic} (bottom) reports the results. For this experiment, we share the same set of $K$ negative examples for both conditional and unconditional MI and therefore we report the upper bound $2 \log K$.

\begin{table*}[t]
\small
\centering
\caption{Accuracy for self-supervised learning on Imagenet-100 (IN100) and on full Imagenet (IN1K), measured by linear evaluation. $x \leftrightarrow y$ denotes standard contrastive matching between views. In DEMI, we use the same base InfoMin architecture but augments the loss function with conditional MI maximization across views. InfoMin (multi) considers $\xup$ just as an additional view and therefore discards conditional MI maximization. All models use a standard Resnet-50 and are trained for 200 epochs. The right part of the table reports transfer learning performance of our model trained on IN1K.}
\begin{tabu}to\linewidth{@{}X[4,l]X[1.4,c]X[c]X[c]*5{X[c]}X[1.7,c]@{}}
\toprule
\textbf{Model} & \textbf{Views} & \textbf{IN100} & \textbf{IN1K} & \textbf{STL10} & \textbf{C10} & \textbf{C100} & \textbf{CARS} & \textbf{CUB} & \textbf{FLOWERS} \\
\cmidrule(r){1-4}
\cmidrule(l){5-10}
SimCLR~\citep{chen2020simple} & $\p \leftrightarrow \xt$ &  - & 66.6 & - & 90.6 & 71.6 & 50.3 & - & 91.2 \\
MocoV2~\citep{chen2020improved} & $\p \leftrightarrow \xt$ &  - & 67.5 & - & - & - & - & - & - \\
InfoMin~\citep{tian2020makes} & $\p \leftrightarrow \xt$ & 74.9 & 70.1 & 96.2 & 92.0 & 73.2 & 48.1 & 41.7 & 93.2 \\
InfoMin (multi) & $\x, \xpp \leftrightarrow \xt$ & 77.2 & 70.2 & 95.9 & 92.6 & 74.5 & 49.2 & 42.1 & 94.7\\
\cmidrule(r){1-4}
\cmidrule(l){5-10}
DEMI & $\x, \xpp \leftrightarrow \xt$ & \textbf{78.6} & \textbf{70.8} & \textbf{96.4} & \textbf{92.8} & \textbf{75.0} & \textbf{51.8} & \textbf{43.6} & \textbf{95.0} \\

\bottomrule
\end{tabu}
\label{tab:c10}
\end{table*}

\subsection{Vision}
\subsubsection{ImageNet}
\paragraph{Setup} \textls[-5]{We study self-supervised learning of image representations using $224{\times}224$ images from ImageNet~\citep{imagenet_cvpr09}. The evaluation is performed by fitting a linear classifier to the task labels using the pre-trained representations only, that is, we fix the weights of the pre-trained image encoder~$f$. We build upon InfoMin~\citep{tian2020makes}. All hyperparameters for training and evaluation are the same as in~\citet{tian2020makes}. All models use a momentum-contrastive memory buffer of $K = 65536$ examples~\citep{chen2020improved}. All models use a Resnet50 backbone and are trained for 200 epochs. We report transfer learning performance by freezing the encoder on STL-10, CIFAR-10
and CIFAR-100 \citep{krizhevsky2009learning}, Stanford Cars \citep{Krause2013CollectingAL}, Caltech-UCSD Birds (CUB) \citep{welinder2010caltech} and Oxford 102 Flowers \citep{zissermannflowers}.}
\vspace{-4mm}
\paragraph{Views} Each input image is independently augmented into two views $x$ and $y$ using a stochastically applied transformation following~\citet{tian2020makes}. This uses random resized crop, color jittering, gaussian blur, rand augment, color dropping, and jigsaw as augmentations. We experiment two ways of creating the subview $\xup$ of $\x$: $\mathtt{cut}$, which applies cutout to $x$, and $\mathtt{crop}$ which is inspired by~\citet{caron2020unsupervised} and consists in cropping the image aggressively and resizing the resulting crops to $96{\times}96$. \textcompress{To do so, we use the $\mathtt{RandomResizedCrop}$ from the $\mathtt{torchvision.transforms}$ module with $\mathtt{s=(0.05, 0.14)}$.}

\vspace{-4mm}
\paragraph{Models} Our baseline, InfoMin, maximizes $I_{\textit{NCE}}(\xp, \xt)$. We also report an enhanced baseline InfoMin (multi), which maximizes $I_{\textit{NCE}}(\xp, \xt) + I_{\textit{NCE}}(\xup, \xt)$ and aims to verify whether additional gains can be obtained by estimating conditional MI rather than just using $\xup$ as an additional view. We recur to $I_{\textit{BO}}$ to estimate the conditional critic\footnote{Although not reported explicitly, we found that $I_{\textit{IS}}$ leads very similar performance with a slightly higher variance across seeds.}. DEMI maximizes four terms: $I_{\textit{NCE}}(\xup; \xt) + I_{\textit{BO}}(\xp; \xt | \xup) + I_{\textit{NCE}}(\x; \xt) + I_{\textit{BO}}(\xup; \xt | \x)$. This  corresponds to maximizing both decompositions of the joint $\mi(\x, \xup; \xt)$. Differently from MI estimation, we found to be important for representation learning to maximize both decompositions, which include $I_{\textit{NCE}}(\x; \xt)$ in the objective. The computation of the conditional MI terms can be efficiently done by reusing the logits of the two unconditional MI (Listing~\ref{lst:code}).
\vspace{-4mm}
\paragraph{Results} Table 1 reports the average accuracy of linear evaluations obtained by 3 pretraining seeds. DEMI obtains 3.7\% improvement (78.6$\pm$0.2) compared to the baseline InfoMin for Imagenet100 (IN100) and 0.7\% (70.8$\pm$0.1) for full Imagenet (IN1K). Although not reported, the $\mathtt{crop}$ strategy performs better than the $\mathtt{cut}$ strategy (which obtains 70.5$\pm$0.1 on average IN1K). One hypothesis is that cutout introduces image patches that do not follow the pixel statistics in the corpus. InfoMin (multi) ablates conditional MI maximization and shows that introducing the additional view is helpful in low-data setting such as IN100, but can only slightly improve performance in IN1K. It is interesting to note that DEMI improves transfer learning performance the most in the fine-grained classification benchmarks CARS and CUB, where it is particularly important to capture detailed information about the input image~\citep{yang2018learning}. This serves as indication that the representations learnt by DEMI can extract more information about each input.

\definecolor{codegreen}{rgb}{0,0.6,0}
\definecolor{codegray}{rgb}{0.5,0.5,0.5}
\definecolor{codepurple}{rgb}{0.58,0,0.82}
\definecolor{backcolour}{rgb}{0.95,0.95,0.92}
\begin{lstlisting}[
float=tp,
language=Python,
floatplacement=tbp,
xleftmargin=2em,
frame=single,
framexleftmargin=1.5em,
backgroundcolor=\color{backcolour},
belowskip=-2\baselineskip,
commentstyle=\color{codegreen},
keywordstyle=\color{magenta},
numberstyle=\tiny\color{codegray},
stringstyle=\color{codepurple},
backgroundcolor=\color{backcolour},
commentstyle=\color{codegreen},
basicstyle=\ttfamily\scriptsize,
breakatwhitespace=false,         
numbers=left,                    
breaklines=true,                 
captionpos=b,                    
keepspaces=true,                 
numbersep=5pt,                  
showspaces=false,                
showstringspaces=false,
showtabs=false,                  
tabsize=2,
label={lst:code},
caption=PyTorch-style pseudo-code for DEMI in InfoMin. We use $I_{\textit{BO}}$ to estimate the critic for conditional MI.]
def compute_demi(x, xp, y, f, f_ema, g, g_bo):
  f_x, f_xp, k_y = f(x), f(xp), g(f_ema(y))
  # NCE heads
  q_x, q_xp = g(f_x), g(f_xp)
  # conditional NCE heads
  q_bo_x, q_bo_xp = g_bo(f_x), g_bo(f_xp)
  # compute NCE critics
  s_x_y = dot(q_x, cat(k_y, memory))
  s_xp_y = dot(q_xp, cat(k_y, memory))
  # compute conditional NCE critics
  s_bo_xp_y = dot(q_bo_xp, cat(k_y, memory))
  s_bo_x_y = dot(q_bo_x, cat(k_y, memory))
  # compute NCE bounds
  nce_x_y = -log_softmax(s_x_y)[0]
  nce_xp_y = -log_softmax(s_xp_y)[0]
  # compute BO estimator
  bo_x_xp = -log_softmax(
    s_xp_y.detach() + s_bo_x_y)[0]
  bo_xp_x = -log_softmax(
    s_x_y.detach() + s_bo_xp_y)[0]
  return (nce_x_y + nce_xp_y + bo_x_xp + bo_xp_x)
\end{lstlisting}

\subsubsection{CIFAR-10}
We also experiment on CIFAR-10 building upon SimCLR~\citep{chen2020improved}, which uses a standard ResNet-50 architecture by replacing the first $7{\times}7$ Conv of stride 2 with $3{\times}3$ Conv of stride 1 and also remove the max pooling operation. In order to generate the views, we use Inception crop (flip and resize to $32{\times}32$) and color distortion. We train with learning rate 0.5, batch-size 800, momentum coefficient of 0.9 and cosine annealing schedule. Our energy function is the cosine similarity between representations scaled by a temperature of 0.5~\citep{chen2020improved}. We obtain a top-1 accuracy of 94.7\% using a linear classifier compared to 94.0\% reported in~\citet{chen2020improved} and 95.1\% for a supervised baseline with same architecture.

\subsection{Dialogue}
\vspace{-1mm}
\paragraph{Setup} We experiment with language modeling task on the Wizard of Wikipedia (WoW) dataset~\citep{dinan2018wizard}. We evaluate our models using  automated metrics and human evaluation. For automated metrics, we report perplexity (ppl), BLEU~\citep{papineni2002bleu}. We report a comprehensive set of metrics in the Appendix (Sec~\ref{app:dialog}). We build upon GPT2~\citep{radford2019language}, and fine-tune it by language modeling (LM) on the dialogue corpus. In addition to the LM loss, we maximize MI between representations of the past and future utterances in each dialogue,~i.e. the predictive coding framework~\citep{elias1955predictive,mcallester2018formal}. We consider past and future in a dialogue as views of the same conversation. Given $L$ utterances $(\x_1, \ldots, \x_L)$, we set $\xt = (x_{k + 1}, \ldots, x_L)$, $\x = (x_1, \ldots, x_k)$ and $\xup = \x_k$, where $(.)$ denotes concatenation and $k$ is randomly chosen between $2 < k < L$. The goal is therefore to imbue representations with information about the future that cannot be solely explained by the most recent utterance $\xup$. The representations of past and future are the state corresponding to the last token in the last layer in GPT2.
\vspace{-3mm}
\paragraph{Models} \textls[-6]{We evaluate our introduced models  against different baselines. GPT2 is a basic small pre-trained model fine-tuned on the dialogue corpus. TransferTransfo ~\citep{wolf2019transfertransfo} augments the standard next-word prediction loss in GPT2 with the next-sentence prediction loss similar to~\citet{devlin2018bert}.
GPT2-MMI follows MMI-bidi~\citep{li2016diversity}; we generate 50 responses from GPT2 and then rank them based on a trained backward model $p_{\textit{GPT2}}(x|y)$. For the InfoNCE baseline, we only maximize the unconditional MI between $x$ and $y$ and sample negative futures from the marginal distribution $p(y)$. 
DEMI maximizes conditional MI by recurring to $I_{\textit{VAR}}$ and using GPT2 itself as the variational approximation. GPT2 is a generative model therefore we can simply sample a set of negative futures from $p_{\textit{GPT2}}(\xt | \xup)$, that is, by restricting the amount of contextual information GPT2 is allowed to consider. To speed up training, the negative sampling of future candidates is done offline. We also tried $I_{\textit{BO}}$ in this setting and obtained similar results.}

\newcommand\winsymb{\ding{51}}
\newcommand\losesymb{\ding{55}}
\newcommand\samesymb{$=$}
\newcommand{\ubold}{\fontseries{b}\selectfont}
\begin{table}[tbp]
\small
\centering
\caption{Perplexity, BLEU and side-by-side human evaluation on WoW~\citep{dinan2018wizard}. H- columns indicate whether DEMI was preferred (\winsymb) or not (\losesymb), or neither (\samesymb) at $\alpha = 0.01$.}
\label{table:results-wow}
\newcommand\myspaceamount{5mm}
\small{

\begin{tabu}to\linewidth{@{}X[1.9,l]X[0.9,c]X[0.8,c]X[0.9,c]X[1.1,c]X[0.85,c]}
\toprule
{\bf Model} & {\bf ppl} & {\bf BLEU} & {\bf H-rel} & {\bf H-hum} & {\bf H-int} \\
\midrule
{GPT2}  & 19.21 & 0.78 & \winsymb & \winsymb & \winsymb \\
{TransferTransfo} & 19.32 & 0.75 & \winsymb & \winsymb & \winsymb  \\
{GPT2-MMI} & 19.30 & 0.65 & \winsymb & \winsymb &  \winsymb \\
\cmidrule{1-1}
InfoNCE & 18.85  & 0.80 & \samesymb & \winsymb & \winsymb \\
DEMI &  \ubold 18.70  & \bfseries 0.82 & \samesymb & \samesymb & \samesymb \\
\cmidrule{1-1}
{Human} &  {--} & {--} & \losesymb & \losesymb & \losesymb \\
\bottomrule
\end{tabu}
}
\end{table}
\vspace{-3mm}
\paragraph{Results} Table~\ref{table:results-wow} shows results on the validation set obtained by 3 pretraining seeds.
For the test set results and sample dialogue exchanges, please refer to the Appendix. The automated metrics indicate that DEMI representations result in higher quality responses. 
We also perform human evaluation on randomly sampled 1000 WoW dialogue contexts. We present the annotators with a pair of candidate responses consisting of InfoNCE, DEMI and baseline responses.
They were asked to compare the pairs regarding interestingness, relevance and humanness, using a 3-point Likert scale~\citep{zhang2019dialogpt}. In Table~\ref{table:results-wow}, we see that overall responses generated by DEMI were strongly preferred to other models but not to the gold response. Bootstrap confidence intervals and p-values (t-test, following \citealp{zhang2019dialogpt}) indicate significant improvements at $\alpha{=}0.01$.

%% file: SEC_related.tex
\section{Related Works}
Representation learning based on MI maximization has been applied in various
domains such as images~\citep{grill2020bootstrap,caron2020unsupervised}, words~\citep{mikolov2013efficient,stratos2018mutual}, graphs~\citep{velivckovic2018deep}, RL~\citep{mazoure2020deep} and videos~\citep{jabri2020space}, exploiting noise-contrastive estimation (NCE)~\citep{gutmann2012noise}, InfoNCE~\citep{oord2018representation} and variational objectives (MINE)~\citep{hjelm2018learning}. InfoNCE have gained recent interest w.r.t. variational approaches due to its lower variance~\citep{song2019understanding} and superior performance in downstream tasks. InfoNCE however can underestimate large amounts of true MI given that it is capped at $\log K$.~\citet{poole2019variational} propose to trade-off between variance and bias by interpolating variational and contrastive bounds.~\citet{song2020multi} propose a modification to InfoNCE for reducing bias where the critic needs to jointly identify multiple positive samples at the same time. Our proposal to scaffold the total MI estimation into a sequence of smaller estimation problems shares similarities with the recent telescopic estimation of density ratio~\citep{rhodes2020telescoping} which is based on variational approximations. Instead, we build upon InfoNCE, propose new results on contrastive conditional MI estimation and apply it to self-supervised representation learning. Other MINE-based approaches of conditional MI estimation can be found in the recent~\cite{mondal2020c}. Our contrastive bound in Eq.~\ref{eq:sep_info_nce} is reminiscent of conditional noise-contrastive estimation~\citep{ceylan2018conditional}, which generalizes NCE for data-conditional noise distributions~\citep{gutmann2012noise}: our result is an interpretation in terms of conditional MI.

%% file: SEC_conclusion.tex
\section{Conclusion}
We decompose the original cross-view MI into a sum of conditional and unconditional MI terms (DEMI). We provide several contrastive approximations to the conditional MI and verify their effectiveness in various domains. Incorporating more than two terms in the decomposition is straightforward and could be investigated in the future. Recent work questioned whether MI maximization itself is at the core of the recent success in representation learning ~\citep{rainforth2018tighter,tschannen2019mutual}. These showed that capturing a larger amount of mutual information between views may not correlate to better downstream performance. Other desirable properties of the representation space may play an important role~\citep{wang2020understanding}.
Although we acknowledge these results, we posit that devising more effective ways to maximize MI will still prove useful in representation learning, especially if paired with architectural inductive biases or explicit regularization methods.

\section*{Acknowledgements}
We would like to acknowledge Jiaming Song and Mike Wu for the insightful discussions and the anonymous reviewers for their helpful comments.

%% file: SEC_appendix.tex
\newpage

\onecolumn

\section{Derivations}
\subsection{Derivation of InfoNCE, $I_{\textit{NCE}}$}
\label{app:proof_nce}
We start from Barber and Agakov's variational lower bound on MI~\citep{barber2003algorithm}. $\mi(\x; \xt)$ can be bounded as follows:
\begin{align}
\mi(\x; \xt) &= \expect_{p(\x, \xt)} \log \frac{p(\xt | \x)}{p(\xt)} \ge \expect_{p(\x, \xt)} \log \frac{q(\xt | \x)}{p(\xt)},
\label{eq:barber-agakov-bound}
\end{align}
where $q$ is an arbitrary distribution. We show that the \INCE{} bound~\citep{oord2018representation} corresponds to a particular choice for the variational distribution $q$ followed by the application of the Jensen inequality. Specifically, $q(\xt | \x)$ is defined by independently sampling a set of examples $\{\xt_1, \ldots, \xt_K\}$ from a proposal distribution $\pi(\xt)$ and then choosing $\xt$ from $\{\xt_1, \ldots, \xt_K\}$ in proportion to the importance weights $w_y = \frac{e^{\cri(\x, \xt)}}{\sum_k e^{\cri(\x, \xt_k)}}$, where $\cri$ is a function that takes $\x$ and $\xt$ and outputs a scalar. In the context of representation learning, $\cri$ is usually a dot product between some representations of $\x$ and $\xt$, e.g.~$f(\x)^T f(\xt)$~\citep{oord2018representation}. The unnormalized density of $\xt$ given a specific set of samples $\xt_{2:K} = \{\xt_2, \ldots, \xt_K\}$ and $\x$ is:
\begin{align}
q(\xt | \x, \xt_{2:K}) = \pi(\xt) \cdot \frac{K\cdot e^{\cri(\x, \xt)}}{e^{\cri(\x, \xt)} + \sum_{k=2}^K e^{\cri(\x, \xt_k)}},
\label{eq:var_dist_unnorm}
\end{align}
where we introduce a factor $K$ which provides ``normalization in expectation''.
By normalization in expectation, we mean that taking the expectation of $q(\xt | \x, \xt_{2:K})$ with respect to resampling of the alternatives $\xt_{2:K}$ from $\pi(\xt)$ produces a normalized density (see Sec.~\ref{app:proof_normalized} for a derivation):
\begin{align}
\bar{q}(\xt | \x) = \expect_{\pi(\xt_{2:K})}[q(\xt | \x, \xt_{2:K})],
\label{eq:var_dist}
\end{align}
where $\pi(\xt_{2:K}) = \prod_{k=2}^K \pi(\xt_k)$. The \INCE{} bound~\citep{oord2018representation} is then obtained by setting the proposal distribution as the marginal distribution, $\pi(\xt) \equiv p(\xt)$ and applying Jensen's inequality, giving:
{\small
\begin{align}
\mi(\x, \xt) & \ge \expect_{p(\x, \xt)} \log \frac{\expect_{p(\xt_{2:K})} q(\xt | \x, \xt_{2:K})}{p(\xt)}   \notag
  \ge \expect_{p(\x, \xt)} \left[ \expect_{p(\xt_{2:K})} \log \frac{p(\xt)\, K \cdot w_{\f}}{p(\xt)} \right] \notag \\
  &=   \expect_{p(\x, \xt)} \left[ \expect_{p(\xt_{2:K})} \log \frac{K \cdot e^{\cri(\x, \xt)}}{e^{\cri(\x, \xt)} + \sum_{k=2}^K e^{\cri(\x, \xt_k)}} \right] \notag \\
  &=   \expect_{p(\x, \xt_1) p(\xt_{2:K})} \bigg[\log \frac{e^{\cri(\x, \xt)}}{\frac{1}{K} \sum_{k=1}^K e^{\cri(\x, \xt_k)}}\bigg] = I_{\textit{NCE}}(\x; \xt | \cri, K) \le \log K,
\label{eq:infonce-app}
\end{align}
}where the second inequality was obtained using Jensen's inequality.

\subsubsection{Derivation of normalized distribution}
\label{app:proof_normalized}
We follow~\citet{cremer2017reinterpreting} to show that $q(\xt | \x) = \mathbb{E}_{\xt_{2:K}\sim \pi(\xt)}[q(\xt | \x, \xt_{2:K})]$ is a normalized distribution:
\begin{alignat}{2}
\small
\int_\x\! q(\xt | \x) \,d\xt &= \int_\xt \mathbb{E}_{\xt_{2:K}\sim \pi(\xt)}\left(\pi(\xt) \frac{e^{\cri(\x, \xt)}}{\frac{1}{K} \left(\sum_{k=2}^K e^{\cri(\x, \xt_k)} + e^{\cri(\x, \xt)}\right)}\right) d\xt \notag \\
&= \int_\xt \pi(\xt) \mathbb{E}_{\xt_{2:K}\sim \pi(\xt)}\left( \frac{e^{\cri(\x, \xt)}}{\frac{1}{K} \left(\sum_{k=2}^K e^{\cri(\x, \xt_k)} + e^{\cri(\x, \xt)}\right)}\right) d\xt \notag \\
&= \mathbb{E}_{\pi(\xt)} \mathbb{E}_{\pi(\xt_{2:K})}\left( \frac{e^{\cri(\x, \xt)}}{\frac{1}{K} \left(\sum_{k=2}^K e^{\cri(\x, \xt_k)} + e^{\cri(\x, \xt)}\right)}\right) \notag \\
&= \mathbb{E}_{\pi(\xt_{1:K})}\left(\frac{e^{\cri(\x, \xt)}}{\frac{1}{K} \sum_{k=1}^K e^{\cri(\x, \xt_k)}}\right) \notag \\
&= K \cdot \mathbb{E}_{\pi(\xt_{1:K})}\left(\frac{e^{\cri(\x, \xt_1)}}{\sum_{k=1}^K e^{\cri(\x, \xt_k)}}\right) \notag \\
&= \sum_{i=1}^K \mathbb{E}_{\pi(\xt_{1:K})}\frac{e^{\cri(\x, \xt_i)}}{\sum_{k=1}^K e^{\cri(\x, \xt_k)}} \notag \\
&= \mathbb{E}_{\pi(\xt_{1:K})}\frac{\sum_{i=1}^K e^{\cri(\x, \xt_i)}}{\sum_{k=1}^K e^{\cri(\x, \xt_k)}} = 1
\end{alignat}

\subsection{Proof of Proposition 1}
\label{app:proof_condmi}
\cnce*
\begin{proof}
We begin with 1., the derivation is as follows:
\begin{align}
\mi(\xp; \xt | \xpp) &= \mathbb{E}_{p(\xpp, \xp, \xt)} \log \frac{p(\xt | \xpp, \xp)}{p(\xt | \xpp)} \ge
\mathbb{E}_{p(\xpp, \xp, \xt)} \log \frac{\bar q(\xt | \xpp, \xp)}{p(\xt | \xpp)} \label{eq:proof_condmi:jensen} \\
& = \mathbb{E}_{p(\xpp, \xp, \xt)} \log \frac{\mathbb{E}_{p(\xt_{2:K} | \xpp)} q(\xt | \xpp, \xp, \xt_{2:K})}{p(\xt | \xpp)} \\ 
& \ge \mathbb{E}_{p(\xpp, \xp, \xt)} \mathbb{E}_{p(\xt_{2:K} | \xpp)} \log \frac{p(\xt | \xpp)\, K \cdot w_{\xt}}{p(\xt | \xpp)} \label{eq:proof_condmi:proposal} \\
& = \mathbb{E}_{p(\xpp, \xp, \xt)} \mathbb{E}_{p(\xt_{2:K} | \xpp)} \log \frac{K \cdot e^{\crip(\xpp, \xp, \xt)}}{\sum_{k=1}^K e^{\crip(\xpp, \xp, \xt_k)}} \\
&= \mathbb{E}_{p( \xpp, \xp, \xt)} \mathbb{E}_{p(\xt_{2:K} | \xpp)} \log \frac{e^{\crip(\xpp, \xp, \xt)}}{\frac{1}{K} \sum_{k=1}^K e^{\crip(\xpp, \xp, \xt_k)}} \\
& = I_{\textit{CNCE}}(\xp; \xt | \xpp, \crip, K),
\end{align}
where in Eq.~\ref{eq:proof_condmi:proposal} we used Jensen's inequality and $p(\xt | \xpp)$ as our proposal distribution for the variational approximation $\bar q(\xt|\xpp, \xp)$.

For 2., we rewrite $I_{\textit{CNCE}}$ by grouping the expectation w.r.t $\xpp$:
\begin{align}
\mathbb{E}_{p(\xpp)}\bigg[\mathbb{E}_{p(\xp, \xt_1 | \xpp)p(\xt_{2:K} | \xpp)} \bigg[ \log \frac{e^{\cri(\xpp, \xp, \xt_1)}}{\frac{1}{K}\sum_{k=1}^K e^{\cri(\xpp, \xp, \xt_k)}}\bigg]\bigg].
\end{align}
Given that both distributions in the inner-most expectation condition on the same $\xpp$, this term has the same form as $I_{\textit{NCE}}$ and therefore the optimal solution is $\crip^*_{\xpp} = \log \frac{p(y | \xp, \xpp)}{p(y|\xpp)} + c_{\xpp}(\xp)$~\citep{ma2018noise}. The optimal $\crip$ for $I_{\textit{CNCE}}$ is thus obtained by choosing $\crip(\xpp, \xp, \xt) = \crip^*_{\xpp}$ for each $\xpp$, giving $\crip^* = \log \frac{p(y | \xp, \xpp)}{p(y|\xpp)} + c(\xp, \xpp)$.

For proving 3., we substitute the optimal critic and take the limit $K \rightarrow \infty$. We have:
\begin{align}
\lim_{K\to\infty} \mathbb{E}_{p(\xpp, \xp, \xt_1)p(\xt_{2:K} | \xpp)} \bigg[
\log\frac{
\frac{p(\xt|\xpp, \xp)}{p(\xt|\xpp)}
}
{
\frac{1}{K}\left(\frac{p(\xt_1|\xpp, \xp)}{p(\xt_1|\xpp)} + \sum_{k=2}^K \frac{p(\xt_k|\xpp, \xp)}{p(\xt_k|\xpp)}\right)}\bigg],
\end{align}
From the Strong Law of Large Numbers, we know that as $\frac{1}{K-1} \sum_{k=1}^{K-1} \frac{p(\xt_k | \xpp, \xp)}{p(\xt_k | \xpp)} \to \mathbb{E}_{p(\xt | \xpp)} \frac{p(\xt | \xpp, \xp)}{p(\xt | \xpp)} = 1$, as $K \to \infty$ a.s., therefore (relabeling $y = y_1$):
\begin{align}
I_{\textit{CNCE}} &\sim_{K\to\infty}  \mathbb{E}_{p(\xpp, \xp, \xt)} \bigg[
\log\frac{\frac{p(\xt |\xpp, \xp)}{p(\xt|\xpp)}}
{\frac{1}{K}\left(\frac{p(\xt|\xpp, \xp)}{p(\xt|\xpp)} + K - 1\right)}
\bigg] \\
& \sim_{K\to\infty} \mathbb{E}_{p(\xpp, \xp, \xt)} \bigg[
\log\frac{p(\xt|\xpp, \xp)}{p(\xt|\xpp)} + \log \frac{K}{\left(\frac{p(\xt|\xpp, \xp)}{p(\xt|\xpp)} + K - 1\right)}\bigg] \\
& \sim_{K\to\infty} I(\xp, \xt |\xpp),
\end{align}
where the last equality is obtained by noting that the second term $\to 0$.
\end{proof}

\subsection{Proof for Proposition 2}
\label{app:proof_variational}
\ivar*

\begin{proof}
For 1., we proceed as follows:
\begin{align}
\mi(\xp; \xt | \xpp) & \ge \mathbb{E}_{p(\p, \xt)} \left[\log \frac{q(\xt | \xpp, \xp) q_\xi(\xt | \xpp)}{p(\xt | \xpp) q_\xi(\xt | \xpp)}\right] \notag \\ 
&= \mathbb{E}_{p(\p, \xt)} \left[\log \frac{q(\xt | \xpp, \xp)}{q_\xi(\xt | \xpp)}\right] - \mathbb{E}_{p(x)}\left[KL(p(\xt | \xpp) \| q_\xi(\xt | \xpp))\right] \notag 
\\
& \ge \mathbb{E}_{p(\p, \xt_1)q_\xi(\xt_{2:K} | \xpp)} \left[\log \frac{e^{\crip(\xpp, \xp, \xt_1)}}{\frac{1}{K} \sum_{k=1}^K e^{\crip(\xpp, \xp, \xt_k)}}\right] - \mathbb{E}_{p(x)} \left[ KL(p(\xt | \xpp) \,\|\, q_\xi(\xt | \xpp))\right], \notag \\
& = I_{\textit{VAR}}(\xp, \xt | \xpp, \crip, \xi, K)
\end{align}
where the last step has been obtained as in Eq.~\ref{eq:proof_condmi:proposal}.

Proving 2. is straightforward by noting that if $q_\xi = p$, $KL(p(y|\xpp) || q_\xi(y|\xpp)) = 0$ and the first term corresponds to $I_{\textit{CNCE}}$.

Proving 3. goes as follows:
\begin{align}
\sup_\crip &\; \mathbb{E}_{p(\xp, \xpp, \xt_1)q_\xi(\xt_{2:K} | \xpp)} \bigg[\log \frac{e^{\crip(\xpp, \xp, \xt_1)}}{\frac{1}{K}\sum_{k=1}^K e^{\crip(\xpp, \xp, \xt_k)}} \bigg] - \mathbb{E}_{p(\xpp)} \bigg[ KL \left( p(\xt | \xpp) \,\|\, q_\xi(\xt | \xpp)\right)\bigg] \\
& = E_{p(\xpp, \xp, \xt_1)q_\xi(\xt_{2:K} | \xpp)} \bigg[ \log \frac{p(y_1 | \xpp, \xp)}{q_\xi(y_1|\xpp)} - \log \frac{p(\xt_1 | \xpp)}{q_\xi(\xt_1 | \xpp)} - \log \frac{1}{K} \sum_{k=1}^K \frac{p(\xt_k | \xp, \xpp)}{q_\xi(\xt_k | \xpp)} \bigg] \\
& = I(\xp, \xt | \xpp) - E_{p(\xpp, \xp, \xt_1)q_\xi(\xt_{2:K} | \xpp)} \bigg[ \log \frac{1}{K} \sum_{k=1}^K \frac{p(\xt_k | \xp, \xpp)}{q_\xi(\xt_k | \xpp)} \bigg] \\
& \to_{K \to \infty} I(\xp, \xt | \xpp).
\end{align}
This is obtained by noting that (1) for any $K$ and $q_\xi$, $\arg \sup_\crip I_{\textit{VAR}} = \log \frac{p(y | \xpp, \xp)}{q_\xi(y | \xp)} + c(\xp, \xpp)$ (because the KL doesn't depend on $\crip$) and (2) the second term in the last line goes to 0 for $K \to \infty$ (a straightforward application of the Strong Law of Large Numbers shows that for samples $\xt_{2:K}$ drawn from $q_\xi(\xt_{2:K} | \xpp)$, we have: $\frac{1}{K} \sum_{k=2}^K \frac{p(\xt_k | \xp, \xpp)}{q_\xi(\xt_k | \xpp)} \to_{K \to \infty} 1$).

\end{proof}

\subsection{Proofs for $I_{\textit{IS}}$}

We will be using the following lemma.
\begin{lemma}\label{lem:slln}
For any $\xpp$, $\xp$ and $y$, and any sequence $\crip_K$ such that $||\crip_K - \crip||_\infty \to_{K \to \infty} 0$:
\begin{align}
   &\lim_{K \to \infty} \mathbb{E}_{p(\xt_{2:K})} \log \frac{K e^{\crip_K(\xpp, \xp, \xt)}}{e^{\crip_K(\xpp, \xp, \xt)} + (K-1)\;{\sum_{k=2}^K w_k e^{\crip_K(\xpp, \xp, \xt_k)}}} \\ 
   &\hspace{6cm} = \lim_{K \to \infty} \mathbb{E}_{p(\xt_{2:K}|\xpp)} \log \frac{K e^{\crip(\xpp, \xp, \xt)}}{e^{\crip(\xpp, \xp, \xt)} + \;{\sum_{k=2}^K e^{\crip(\xpp, \xp, \xt_k)}}},
\end{align}
where $w_k = \frac{\exp \cri^*(\xpp, \xt_k)}{\sum_{k=2}^K \exp \cri^*(\xpp, \xt_k)}$, for $\cri^*(\xpp, \xt_k) = \arg \sup_\cri I_{\textit{NCE}}(\xpp, \xt | \cri, K) = \log \frac{p(y_k|\xpp)}{p(y_k)}$.
\end{lemma}
\begin{proof}
We see that almost surely, for $\xt_{2:K} \sim p(\cdot)$:
\begin{equation}
\sum_{k=2}^K w_k e^{\crip_K(\xpp, \xp, \xt_k)} = \frac{\frac{1}{K - 1} \sum_{k=2}^{K} \frac{p(y_k|\xpp)}{p(y_k)} e^{\crip_K(\xpp, \xp, \xt_k)}}{\frac{1}{K - 1}\sum_{k=2}^K \frac{p(y_k|\xpp)}{p(y_k)}} \to_{K \to \infty} \mathbb{E}_{p(y | \xpp)} e^{\crip(\xpp, \xp, y)},
\end{equation}
where we applied the Strong Law of Large Numbers to the denominator. 

For the numerator, we write:
\begin{align*}
    \frac{1}{K - 1} \sum_{k=2}^{K} \frac{p(y_k|\xpp)}{p(y_k)} e^{\crip_K(\xpp, \xp, \xt_k)} = \frac{1}{K - 1}& \sum_{k=2}^{K} \frac{p(y_k|\xpp)}{p(y_k)} e^{\crip(\xpp, \xp, \xt_k)} \\ 
    &+ \frac{1}{K - 1} \sum_{k=2}^{K} \frac{p(y_k|\xpp)}{p(y_k)} (e^{\crip_K(\xpp, \xp, \xt_k)} - e^{\crip(\xpp, \xp, \xt_k)})
\end{align*}
and note that the first term is the standard IS estimator using $p(y_k)$ as proposal distribution and tends to $\mathbb{E}_{p(y | \xpp)} e^{\crip(\xpp, \xp, \xt)}$ from the Strong Law of Large Numbers, while the second term goes to $0$ as $\crip_K$ tends to $\crip$ uniformly.

This gives $\lim_{K \to \infty} \mathbb{E}_{p(\xt_{2:K})} \log \frac{K e^{\crip_K(\xpp, \xp, \xt)}}{e^{\crip_K(\xpp, \xp, \xt)} + (K-1)\;{\sum_{k=2}^K w_k e^{\crip_K(\xpp, \xp, \xt_k)}}} = \log \frac{e^{\crip(\xpp, \xp, \xt)}}{\;\mathbb{E}_{p(y | \xpp)} e^{\crip(\xpp, \xp, \xt)}}$.

Following the same logic (without the importance-sampling) demonstrates that: 
\begin{align*}
\lim_{K \to \infty} \mathbb{E}_{p(\xt_{2:K}|\xpp)} \log \frac{K e^{\crip(\xpp, \xp, \xt)}}{e^{\crip(\xpp, \xp, \xt)} + \;{\sum_{k=2}^K e^{\crip(\xpp, \xp, \xt_k)}}} = \log \frac{e^{\crip(\xpp, \xp, \xt)}}{\;\mathbb{E}_{p(y | \xpp)} e^{\crip(\xpp, \xp, \xt)}},
\end{align*}
which concludes the proof.
\end{proof}

\label{app:proof_is}
\is*
\begin{proof}
By applying Lemma~\ref{lem:slln} with $\crip_K = \crip$, we know that for any $\crip$: 
\begin{align*}
    \lim_{K \to \infty} I_{\textit{IS}}(\xp; \xt | \xpp, \crip, K) &= \lim_{K \to \infty} \mathbb{E}_{p(\xpp, \xp, y) p(\xt_{2:K}|\xpp)} \log \frac{K e^{\crip(\xpp, \xp, \xt)}}{e^{\crip(\xpp, \xp, \xt)} + \;{\sum_{k=2}^K e^{\crip(\xpp, \xp, \xt_k)}}}.
\end{align*}
In particular, the RHS of the equality corresponds to $\lim_{K \to \infty} I_{\textit{CNCE}}(\xp, y | \xpp, \crip, K)$. That quantity is smaller than $I(\xp, \xt | \xpp)$, with equality for $\crip = \crip^*$. This guarantees that:
\begin{align}
\lim_{K \to \infty} \sup_{\crip} I_{\textit{IS}}(\xp; \xt | \xpp, \crip, K) \geq \lim_{K \to \infty} I_{\textit{IS}}(\xp; \xt | \xpp, \crip^*, K) = I(\xp, \xt | \xpp).
\end{align}
We now prove the reverse inequality. We let $2 \epsilon = \lim_{K \to \infty} \sup_{\crip} I_{\textit{IS}}(\xp; \xt | \xpp, \crip, K) - I(\xp, \xt | \xpp)$, and assume towards a contradiction that $\epsilon > 0$. We know that: 
$$\exists K_0, \quad \forall K \geq K_0, \quad \sup_{\crip} I_{\textit{IS}}(\xp; \xt | \xpp, \crip, K) \geq I(\xp, \xt | \xpp) + \epsilon.$$
Now, $\forall K \geq K_0$, let $\crip_K$ be such that:
$$I_{\textit{IS}}(\xp; \xt | \xpp, \crip_K, K) \geq \sup_{\crip} I_{\textit{IS}}(\xp; \xt | \xpp, \crip, K) - \frac{\epsilon}{2},$$

and thus: $\forall K \geq K_0, I_{\textit{IS}}(\xp; \xt | \xpp, \crip_K, K) \geq I(\xp, \xt | \xpp) + \frac{\epsilon}{2}$.

Since $\crip_K \in \mathbb{R}^{|\mathcal{X}| \times |\mathcal{X}| \times |\mathcal{Y}|}$, $\{\crip_K\}_{K \geq K_0}$ contains a subsequence that converges to a certain $\crip_{\infty} \in \bar{\mathbb{R}}^{|\mathcal{X}| \times |\mathcal{X}| \times |\mathcal{Y}|}$. Without loss of generality, we assume that $\forall K, \forall\xpp, \forall\xp, \mathbb{E}_{p(y)} [\crip_K(\xpp, \xp, y)] = 0$ which implies that $\mathbb{E}_{p(y)} [\crip_\infty(\xpp, \xp, y)] = 0$ (similarly to $I_{\textit{NCE}}$, $I_{\textit{IS}}$ is invariant to functions of $(\xpp,\xp)$ added to $\crip$).

In particular, this guarantees that $||\crip_\infty||_\infty < \infty$. Otherwise, we would have $\crip_\infty(\xpp, \xp, y) = -\infty$ for a given $y$, which would then imply $I_{\textit{IS}}(\xp; \xt | \xpp, \crip_\infty, K) = -\infty$ and give a contradiction.

We can now apply Lemma~\ref{lem:slln} to  $\{\crip_K\}$ and $\crip_\infty$ to show that $\lim_{K \to \infty} I_{\textit{IS}}(\xp; \xt | \xpp, \crip_K, K) = \lim_{K \to \infty} I_{\textit{CNCE}}(\xp, y | \xpp, \crip_{\infty}, K)$, and get a contradiction: the first term is larger than $I(\xp, \xt | \xpp) + \frac{\epsilon}{2}$ while the second is smaller than $I(\xp, \xt | \xpp)$.
\end{proof}

\subsection{Proof for $I_\textit{BO}$}
\label{proof:bo}
\bo*
\begin{proof}
To prove 1., it suffices to follow the proof for $I_{\textit{NCE}}$ (Sec.~\ref{app:proof_nce}). To prove 2., we set $\eta(\xup, \x, \xt) = \cri^*(\xup, \xt) + \phi(\xup, \x, \xt_1)$.~\citet{ma2018noise} show that $\eta^*(\xup, \x, \xt) = \log \frac{p(\xt | \xup, \x)}{p(\xt)} + c^\eta(\xup, \x)$, for any $K$. Knowing that $\cri^*(\xup, \xt) = \log \frac{p(\xt | \xup)}{p(\xt)} + c^\cri(\xup)$ is a constant in the maximization problem, simple algebra shows that $\phi^*(\xup, \x, \xt) = \log \frac{p(\xt | \xup, \x)}{p(\xt | \xup)} + c(\xup, \x)$.
\end{proof}

\subsection{Synthetic Experiments}
\label{app:synthetic}
Here, we provide details for Sec.~\ref{sec:empirical_validation}. In this experiment, each $\xp$, $\xpp$ and $\xt$ are 20-dimensional. For each dimension, we sampled $(\xp_i , \xpp_i, \xt_i)$ from a correlated Gaussian with mean 0 and covariance matrix $\mathtt{cov}_i$.
For a given value of MI, \texttt{mi} = $\{5, 10, 15, 20\}$, we sample covariance matrices
$\mathtt{cov}_i =\text{\texttt{sample\_cov(mi${}_i$)}}$, such that $\sum_i \mathtt{mi}_i = \mathtt{mi}$, $\mathtt{mi}_i>0$ chosen at random. We optimize the bounds by stochastic gradient descent (Adam, learning rate $5\cdot 10^{-4}$).
All encoders $f$ are multi-layer perceptrons with a single hidden layer and ReLU activation. Both hidden and output layer have size 100.

InfoNCE computes:
{\small
\begin{align}
\mathbb{E}_p\left[ \log \frac{e^{f([\xp, \xpp])^T f(\xt)}}{e^{f([\xp, \xpp])^T f(\xt)} + \sum_{k=2}^K e^{f([\xp, \xpp])^T f(\xt_k)}}\right] + \log K,\;\; \xt_{2:K} \sim p(\xt),\notag
\end{align}
}where the proposal is the marginal distribution $p(\xt)$, $E$ is chosen to be a dot product between representations, $\mathbb{E}_p$ denotes expectation w.r.t. the known joint distribution $p(\xp, \xpp, \xt)$ and is approximated with Monte-Carlo, $[\xp, \xpp]$ denotes concatenation and $f$ is a 1-hidden layer MLP.

DEMI computes:
{\small
\begin{align}
\mathbb{E}_{p(\xpp, \xp, \xt)p(\xt_{2:K / 2})}&\left[ \log \frac{e^{f(\xpp)^T f(\xt)}}{e^{f(\xpp)^T f(\xt)} + \sum_{k=2}^{K / 2} e^{f(\xpp)^T f(\xt_k)}} \right] + \\
& \mathbb{E}_{p(\xpp, \xp, \xt)p(\xt_{2:{K / 2}} | \xpp)}\left[ \log \frac{e^{f([\xpp, \xp])^T f(\xt)}}{ e^{f([\xpp, \xp])^T f(\xt)} + \sum_{k=2}^{K / 2} e^{f([\xpp, \xp])^T f(\xt_k)}} \right] + 2 \log K / 2 \notag
\end{align}
}where $f(\x)$ is just $f([\x, \mathbf{0}])$ in order to re-use MLP parameters for the two terms. The negative samples of the conditional MI term come from the conditional distribution $p(\xt|\xpp)$, which is assumed to be known in this controlled setting.
We maximize both lower bounds with respect to the encoder $f$. We report pseudo-code for $\mathtt{sample\_cov}$ in Listing~\ref{lst:cov}, used to generate $3{\times}3$ covariance matrices for a fixed $I(\{\x,\xup\};y)$ and uniformly sampled $\alpha = I(\x;\xt) / I(\{\x, \xup\};y)$.

\newcommand{\mytt}{\begingroup
\catcode`_=12 \domytt}
\newcommand\domytt[1]{\text{\texttt{#1}}\endgroup}

\definecolor{codegreen}{rgb}{0,0.6,0}
\definecolor{codegray}{rgb}{0.5,0.5,0.5}
\definecolor{codepurple}{rgb}{0.58,0,0.82}
\definecolor{backcolour}{rgb}{0.95,0.95,0.92}
\begin{lstlisting}[
float=tp,
language=Python,
floatplacement=tbp,
xleftmargin=2em,
frame=single,
framexleftmargin=1.5em,
backgroundcolor=\color{backcolour},
belowskip=-2\baselineskip,
commentstyle=\color{codegreen},
keywordstyle=\color{magenta},
numberstyle=\tiny\color{codegray},
stringstyle=\color{codepurple},
backgroundcolor=\color{backcolour},
commentstyle=\color{codegreen},
basicstyle=\ttfamily\scriptsize,
breakatwhitespace=false,         
numbers=left,                    
breaklines=true,                 
captionpos=b,                    
keepspaces=true,                 
numbersep=5pt,                  
showspaces=false,                
showstringspaces=false,
showtabs=false,                  
tabsize=2,
label={lst:cov},
caption=Pseudo-code for covariance sampling in the synthetic experiment.]
def sample_cov(mi):
  alpha = random.uniform(0.1, 0.9)
  params = random.normal(0, $\mathbb I_6$)
  # use black box optimizer (Nealder-Mead) to determine opt_params
  opt_param = $\arg\min_x \mytt{residual}(\mytt{params}, \mytt{mi}, \alpha)$
  return project_posdef(opt_params)
  
def project_posdef(x):
  # project x $\in\mathbb R^6$ to a positive definite 3x3 matrix
  cov = zeros(3, 3)
  cov[tril_indices(3)] = x
  cov /= column_norm(cov)
  return dot(cov, cov.T)
  
def analytical_mi(cov):
  # compute analytical MI of 3 covariate Gaussian variables
  cov_01 = cov[:2, :2]
  cov_2 = cov[2:3, 2:3]
  mi_xp_xpp_y = 0.5 * (log(det(cov_01)) + log(det(cov_2)) - log(det(cov)))
  cov_1 = cov[1:2, 1:2]
  cov_23 = cov[1:, 1:]
  mi_xp_y = 0.5 * (log(det(cov_1)) + log(det(cov_2)) - log(det(cov_23)))
  return mi_xp_xpp_y, mi_xp_y
  
def residual(x, mi, $\alpha$):
  # penalize difference between analytical mi and target mi, $\alpha\,\mytt{mi}$
  cov = project_posdef(x)
  mi_xp_y, mi_xp_y = analytical_mi(cov)
  return (mi_xp_xpp_y - mi) ** 2 + (mi_xp_y - $\alpha$ * mi) ** 2
\end{lstlisting}

\enlargethispage{2em}
 \begin{table*}[htbp]
 \centering
    \caption{A sample dialogue between speaker $A$ and speaker $B$ from the Wizard of Wikipedia dataset. The four rows from top to bottom are: (1) $x$: the ``past" dialogue up to utterance $k$ (2) $y$: the ground-truth utterance for the next turn $k+1$ (3) $\xt_{1:N}$: future candidates  sampled from the ``restricted context" future distribution  $p(\xt | \xup)$. These candidates correspond to the set of \textbf{hard} negatives that are closely related to the conversation. (4) $y'_{1:N}$: future candidates sampled randomly from the dataset.  We can see that candidates $\xt_{1:N}$ are semantically close but incoherent w.r.t to the dialogue history as they were conditioned solely on the immediate past utterance $\xup$. However, we can notice that candidates $y'_{1:N}$ are semantically distant from $\p$ as they were sampled randomly from the data distribution.
    The highlighted text in green correspond to the topic of the conversation. Speaker $B$ mentions that they have never done either parachuting or skydiving. $B_1$ corresponds to the utterance generated based on the restricted context $\xup$. The utterance is on-topic but completely contradictory to what speaker $B$ has said in the past. On the other hand $B'_1$ is randomly sampled  from other dialogues. We can observe that  the utterance is clearly irrelevant to the conversation. 
    }
    \label{table:qualitative-1}
    \begin{tabu}to\linewidth{@{}X[1.2,l]X[0.2,l]X[6,l]@{}}
    \toprule
    $\p$ & $A$: & I like \colorbox{emerald}{parachuting or skydiving}.\\
      & \colorbox{lightcoral}{$\bm{B}$}: & \colorbox{gold(web)(golden)}{I've never done either} but they sound terrifying, not a fan of heights.\\
      & $A$: &But it is interesting game. This first parachute jump in history was made by Andre Jacques.\\
      & $\bm{B}$: & Oh really ? Sounds like a french name, what year did he do it ?\\
      & $A$: & It done in October 22 1797. They tested his contraption by leaping from a hydrogen balloon.\\
      & $\bm{B}$: & Was he successful or did he kick the bucket off that stunt? \\
      & $A$: & I think its a success. The military developed parachuting tech.\\
      \midrule
      $\xt  \sim p(\xt | \xup)$ & $B_{gt}$ & Yeah nowadays they are a lot more stable and well made. \\
    \midrule
    $\xt_{1:N} \sim p(\xt | \xup)$
    & \colorbox{lightcoral}{$B_1$}: &  That is great. \colorbox{gold(web)(golden)}{I've been skydiving for days now}. How is it ?\\
    & $B_2$: & Oh I have never flown but I'm glad to know.\\
    & $B_3$: & I've been dying for it since I was a kid. \\
    & $B_4$: & Yes, that is why NASA had an advanced mechanics tech for months. \\
      & $B_5$: & I went parachuting last Sunday and enjoyed it. \\
        \midrule
    $y'_{1:N} \sim p(\xt)$
    & \colorbox{lightcoral}{$B'_1$}: & \colorbox{lightcornflowerblue}{I think science fiction is an amazing genre for anything}\\
    & $B'_2$: & Can you imagine the world without internet access ? \\
    & $B'_3$: & I am just finishing my university course and I will be a qualified pharmacist. \\
    & $B'_4$: & I don't know how to be romantic. I have trouble expressing emotional attraction. \\
    & $B'_5$: & I think Krav Maga is a martial art sport. That 's the reason I picked it .
    \\
    \bottomrule
    \end{tabu}
    \end{table*}

\section{Experiments on Dialogue}
\label{app:dialog}
\subsection{DEMI Details}
The optimization of the DEMI requires the specification of a critic. Following previous work~\citep{oord2018representation,hjelm2018learning}, we implement the critic by a dot product between representations of the past $f(x)$ and those of the future $f(y)$. We obtain $f_{x}$, $f_{y}$ by running a forward pass with the GPT2 model on the words from the past and the future separately and by taking the state of the last layer of the GPT2 corresponding to the last token in the past and the future respectively.

For all DEMI terms, given the past, the model is trained to pick the ground-truth future among a set of $N$ future candidates. This candidate set includes the ground-truth future and $N-1$ negative futures drawn from different proposal distributions. 
To compute $I_{\textit{NCE}}(x; \xt)$, we consider the ground truth future of each sample in the batch as a negative candidate for the other samples in the same batch. Using this approach, the number of candidates $N$ is equated to the batch size. This ensures that negative samples are sampled from the marginal distribution $p(y)$.
To compute the conditional MI boud $I_{\textit{CNCE}}(\x; \xt | \xup)$, we sample negative futures $p(\xt | \xup)$ by conditioning the GPT2 model on the most recent utterance in the past $\xup$.  
\subsection{Dataset}
\textbf{Wizard of Wikipedia} \cite{dinan2018wizard}
consists of \num{20365} dialogues where each dialogue in the conversation is about a specific topic. There are two participants in the conversation: the wizard and the apprentice. The apprentice is a curious learner who is eager to know more about a particular topic. However, the wizard is a knowledgeable expert who tries to inform the apprentice about the topic. In our experiments, we used the valid data ``unseen valid" that includes topics that do not overlap with the train data and the test data. 
Detailed statistics of the dataset are presented in Table~\ref{tab:wiki}. 

    \begin{table*}[!htbp]
        \centering
        \caption{Statistics of the Wizard of Wikipedia dataset}
        \label{tab:wiki}
        \vspace*{0.1in}
        \begin{tabu}to.5\linewidth{@{}X[l]S[table-format=6.0]S[table-format=4.0]S[table-format=4.0]@{}}
         \toprule
         & \textbf{\# Train} & \textbf{\# Valid} & \textbf{\# Test}\\
         \midrule
        Number of utterances &  166787 & 8806 & 8782 \\
        Number of dialogues &   18430 & 967 & 968 \\
        Number of topics  &  1247 & 300 & 58 \\
        Average turns per dialog  &   9 & 9 & 9 \\
         \bottomrule
        \end{tabu}
    \end{table*}

\subsection{Experimental Setup}
 Given memory constraints, all the proposed models are trained with a batch size of 5 per GPU, considering up to three utterances for the future and five utterances in the past. All the models are trained on 2 NVIDIA V100s. The models early-stop in the 4th epoch. 
We use the Adam optimizer with a learning rate of \num{6.25e-5}, which we linearly decay to zero during training. Dropout is set to 10\% on all layers. \INCE{}/DEMI  terms are weighted with a factor 0.1 in the loss function. We varied the factor from 0.1 to 1 and 0.1 was chosen based on the best results on the validation set. During inference, we use nucleus sampling \cite{holtzman2019curious} with $p=0.9$ for all models.

\subsection{Additional Automated metrics}
\paragraph{Repetition} The word repetition metrics aim at testing the model's performance in generating responses while avoiding artificial repetitions. We employ the repetition metrics presented in~\citet{welleck2019neural}: \textbf{seq-rep-}$n$, \textbf{rep}, \textbf{wrep} and \textbf{uniq}. These metrics are defined based on the amount of repetitions in the generations. \textbf{seq-rep-}$n$ measures the portion of duplicate n-grams in a generated sequence:
\begin{align}\label{repetition}
\textbf{seq-rep-}n = 1 - \frac{\left|\text{unique n-grams}(w_{1:N})\right|}{\left|\text{n-grams}\right|} 
\end{align}
where $w_{1:N}$ is the generated utterance. We report \textbf{seq-rep-avg} which averages over $n \in \{2,3,4,5,6\}$. \textbf{rep} measures the fraction of tokens that occur in previous tokens, \textbf{uniq} counts the number of unique tokens on the validation set. Please refer to~\citep{welleck2019neural,Li2019DontST} for more information about these metrics. 

\paragraph{Distinct-n} The metric is  derived from ~\citet{li2016diversity}. It is defined as the number of unique $n$-grams, normalized by the total number of $n$-grams of tested sentences.

\paragraph{Entropy-$n$} We employ the entropy metric from ~\citet{zhang2018generating} which aims to fix the problem of frequency difference of n-grams in Distinct-n by reflecting how evenly the empirical n-gram distribution is for each given sentence.

Results on the test set and the valid set are presented in Table \ref{table:results-wow-app} and Table \ref{table:results-wow-app-valid} respectively. 

\begin{table*}[tbp]
\robustify{\bfseries}
\small
\centering
    \caption{Results for perplexity, sequence-level metric, token-level metrics, BLEU and diversity metrics on the test data of the Wizard of Wikipedia dataset. Results demonstrate that the proposed \INCE{} and DEMI bounds achieve lower perplexity, reduce next-token
    repetition and increase the number of unique next-tokens  compared to the baselines GPT2, GPT2-MMI and TransferTransfo. Note that our results are not directly comparable with \citet{Li2019DontST} as their model is trained from scratch on a not publicly available Reddit-based corpus.} 
    \label{table:results-wow-app}
    \newcommand\myspaceamount{5mm}
    \newcommand\mycolsep{\hspace{\myspaceamount}}
    \setlength{\tabcolsep}{1pt}
    \newcolumntype A{S[table-format=1.3,round-mode=places,round-precision=3,detect-weight,table-column-width=10mm, tight-spacing=false]}
    \newcolumntype B{S[table-format=2.2,round-mode=places,round-precision=2,detect-weight,table-column-width=10mm, tight-spacing=false]}
    \newcolumntype C{S[table-format=4.0,round-mode=places,round-precision=0,detect-weight, tight-spacing=false]}
    \newcolumntype D{S[table-format=-1.2,round-mode=places,round-precision=2,detect-weight,table-column-width=10mm, tight-spacing=false]}
    \small{
    \begin{tabu}to .7\linewidth{@{}X[1,l]BAAACAAAA@{}}
    \toprule
    \multicolumn1{@{}l}{{\bf Model}}
    & \multicolumn1c{{\bf ppl}}
    & \multicolumn1c{{\bf seq-rep-avg}} 
    & \multicolumn1c{{\bf rep}}
    & \multicolumn1c{{\bf wrep}}
    & \multicolumn1c{{\bf uniq}}
    & \multicolumn1c{{\bf dist-1}}
    & \multicolumn1c{{\bf dist-2}} 
    & \multicolumn1c{{\bf BLEU}} 
    & \multicolumn1c{{\bf Entropy-4}} \\

    \midrule
    {GPT2}  & 19.24 & 0.064 &  0.130 &  0.132  &  7393  & 0.064  & 0.392 & 0.775 & 0.095\\
    {TransferTransfo}  &  19.33  & 0.078 &  0.134 &  0.132  &  7735  &   0.058 & 0.386 & 0.752 & 0.084\\
     GPT2-MMI & 19.35  &  0.070 &  0.129 &  0.135 & 7623  &  0.052 & 0.384 & 0.740 &    0.092\\
    InfoNCE & 18.88  &  0.065 &  0.126 &  0.131 & 8432  &  0.065 & 0.390 & 0.799 &  0.107 \\
    DEMI & \bfseries 18.66 & \bfseries 0.050   &  \bfseries 0.120 & \bfseries 0.128  & \bfseries 8666 &  \bfseries 0.070 & \bfseries 0.405& \bfseries 0.810  & \bfseries 0.108\\
    \cmidrule(r){1-1}
      {Ground Truth}    &  {--}& 0.052 & 0.095 & {--}  & 9236 &  0.069 & 0.416 & {--} & 0.11 \\
    \bottomrule
    \end{tabu}
    }
    \end{table*}

\begin{table*}[tbp]
\robustify{\bfseries}
\small
\centering
    \caption{Results for perplexity, sequence-level metric, token-level metrics, BLEU and diversity metrics on the valid data of the Wizard of Wikipedia dataset.}
    \label{table:results-wow-app-valid}
    \newcommand\myspaceamount{5mm}
    \newcommand\mycolsep{\hspace{\myspaceamount}}
    \setlength{\tabcolsep}{1pt}
    \newcolumntype A{S[table-format=1.3,round-mode=places,round-precision=3,detect-weight,table-column-width=10mm, tight-spacing=false]}
    \newcolumntype B{S[table-format=2.2,round-mode=places,round-precision=2,detect-weight,table-column-width=10mm, tight-spacing=false]}
    \newcolumntype C{S[table-format=4.0,round-mode=places,round-precision=0,detect-weight, tight-spacing=false]}
    \newcolumntype D{S[table-format=-1.2,round-mode=places,round-precision=2,detect-weight,table-column-width=10mm, tight-spacing=false]}
    \small{
    \begin{tabu}to .7\linewidth{@{}X[1,l]BAAACAAAA@{}}
    \toprule
    \multicolumn1{@{}l}{{\bf Model}}
    & \multicolumn1c{{\bf ppl}}
    & \multicolumn1c{{\bf seq-rep-avg}} 
    & \multicolumn1c{{\bf rep}}
    & \multicolumn1c{{\bf wrep}}
    & \multicolumn1c{{\bf uniq}}
    & \multicolumn1c{{\bf dist-1}}
    & \multicolumn1c{{\bf dist-2}} 
    & \multicolumn1c{{\bf BLEU}} 
    & \multicolumn1c{{\bf Entropy-4}} \\

    \midrule
    {GPT2}  & 19.21 & 0.066 &  0.136 &  0.134  &  7259  & 0.083  &    0.432 & 0.78 & 0.106\\
    {TransferTransfo}  &  19.32  & 0.074 &  0.134 &  0.133  &  7223  &   0.082 & 0.431 & 0.75 & 0.106\\
     GPT2-MMI & 19.30  &  0.065 &  0.126 &  0.130 & 7351  &  0.082 & 0.425 & 0.65 &   0.106\\
    InfoNCE & 18.85  &  0.065 &  0.126 &  0.126 & 7300  & 0.091 & 0.046 & 0.8 &  0.107 \\
    DEMI & \bfseries 18.70 & \bfseries 0.054   &  \bfseries 0.130 & \bfseries 0.120  & \bfseries 7375 & \bfseries 0.077 & \bfseries 0.0467 & \bfseries 0.82  & \bfseries 0.108\\
    \cmidrule(r){1-1}
      {Ground Truth}    &  {--}& 0.052 & 0.095 & {--}  & 9236 &  0.069 & 0.416 & {--} & 0.11 \\
    \bottomrule
    \end{tabu}
    }
    \end{table*}

 \begin{table*}[]
\renewcommand{\arraystretch}{1.3}
\linespread{0.8}\selectfont
\centering
\caption{Selected responses to the same context from different methods fine-tuned on the Wizard of Wikipedia datasets.}
\vspace*{0.1in}
\newcommand\myhangindent{\hangindent=1.4em}
{\footnotesize
\begin{tabular}{@{}>{\hangindent=2em}p{6.7cm}>{\hangindent=6em}p{9.7cm}@{}}
\toprule
    {{\bf Context}} & { {\bf Generated Responses}} \\
    \midrule
    \multirow{4}{=}{%
        \myhangindent
        {\bf A:} I recently took up archery and I think I love it . Do you know anything about Archery ?\endgraf
        \myhangindent
        {\bf B:} Archery is a sport / skill of using a bow to propel arrows. \endgraf
        \myhangindent
        {\bf A:} And a great sport it is . Do you know where archery originated from ?}
      & {\bf DEMI}: Archery originated from ancient Mesopotamia, and is actually practiced internationally in over 200 countries.  \\
      & {\bf InfoNCE}: I am not sure but many traditions and principles arose during the Mesolithic and Paleolithic era. \\
      & {\bf TransferTransfo }: Yep , you just use it for skill and using it to shoot arrows.
      \\
 & {\bf GPT2}: I don't know, but I know that the old French called it archer's art.\\
 \midrule
    \multirow{4}{=}{%
        \myhangindent
{\bf A:} What is your motivation?\endgraf
        \myhangindent
{\bf B:} That's a tough question. I 'd say my own progress. I like seeing how much I improve with something. You ?\endgraf
        \myhangindent
{\bf A:} I am retired now.   Are you familiar with the work of Mehr and Meyer, well known psychologists ?\endgraf
        \myhangindent
{\bf B:} They said that " Motivation is a word that is part of the popular culture as few other psychological concepts are . "   Do you agree with what they said ?
}
      & {\bf DEMI}: Yes, definitely! You should check out Mark Twain's famous motivational book ! There are theories that point out how important it is to meditate. \\
      & {\bf InfoNCE}: They said that in psychology research, theories and evidence, participants who enjoyed continuous experience could be more likely to be successful. \\
      & {\bf TransferTransfo}: Absolutely, I disagree. Are you an author or a commentator? \\
      &{\bf GPT2}: That's awesome. So , what types of differences exist between ``good " motivation and bad? \\
\midrule
\end{tabular}
}
\label{tab:SampleResponses}
\end{table*}

\begin{table}[h]
\small
\begin{minipage}{\linewidth}
\centering
\begin{tabular}{@{}lrlrlll@{}}
\toprule
{} &  DEMI\_wins &       DEMI\_CI &  baseline\_wins &   baseline\_CI &     pairwise\_CI &        $p$ \\
Baseline          &            &               &                &               &                              &         \\
\midrule
GPT2          &    0.48726 &  (0.44, 0.53] &        0.28662 &  (0.25, 0.32] &    (0.13, 0.27]            * &  <0.001 \\
GPT2-MMI       &    0.65833 &   (0.6, 0.71] &        0.16250 &  (0.12, 0.21] &     (0.4, 0.58]            * &  <0.001 \\
TransferTransfo      &    0.46888 &  (0.43, 0.51] &        0.30043 &  (0.26, 0.34] &    (0.09, 0.24]            * &  <0.001 \\
InfoNCE       &    0.41711 &  (0.38, 0.46] &        0.36748 &  (0.33, 0.41] &   (-0.03, 0.13]              &  0.0905 \\
gold\_response &    0.22679 &  (0.19, 0.26] &        0.54325 &   (0.5, 0.59] &  (-0.39, -0.25]            * &  <0.001 \\
\bottomrule
\end{tabular}
\caption{Which response is more \textit{relevant}?}
\end{minipage}
\begin{minipage}{\linewidth}
\centering
\begin{tabular}{@{}lrlrllll@{}}
\toprule
{} &  DEMI\_wins &       DEMI\_CI &  baseline\_wins &   baseline\_CI &     pairwise\_CI &        $p$ \\
Baseline          &            &               &                &               &                              &         \\
\midrule
GPT2          &    0.45084 &  (0.41, 0.49] &        0.32636 &  (0.29, 0.37] &     (0.05, 0.2]            * &  <0.001 \\
GPT2-MMI       &    0.61734 &  (0.56, 0.67] &        0.18393 &  (0.14, 0.23] &    (0.34, 0.53]            * &  <0.001 \\
TransferTransfo      &    0.43617 &   (0.4, 0.48] &        0.35000 &  (0.31, 0.39] &    (0.01, 0.16]            * &  0.0028 \\
InfoNCE       &    0.44630 &  (0.41, 0.49] &        0.34515 &  (0.31, 0.38] &    (0.03, 0.17]            * &  <0.001 \\
gold\_response &    0.22164 &  (0.19, 0.26] &        0.56608 &  (0.52, 0.61] &  (-0.41, -0.28]           * &  <0.001 \\
\bottomrule
\end{tabular}
\caption{Which response is more \textit{humanlike}?}
\end{minipage}
\begin{minipage}{\linewidth}
\small
\centering
\begin{tabular}{@{}lrlrllc@{}}
\toprule
{} &  DEMI\_wins &       DEMI\_CI &  baseline\_wins &   baseline\_CI &     pairwise\_CI &        $p$ \\
Baseline          &            &               &                &               &                              &         \\
\midrule
GPT2          &    0.56157 &   (0.52, 0.6] &        0.21444 &  (0.18, 0.25] &    (0.28, 0.42]            * &  <0.001 \\
GPT2-MMI       &    0.68750 &  (0.63, 0.74] &        0.12292 &  (0.09, 0.16] &    (0.48, 0.65]            * &  <0.001 \\
TransferTransfo      &    0.51931 &  (0.48, 0.56] &        0.24571 &  (0.21, 0.28] &    (0.21, 0.34]      * &  <0.001 \\
InfoNCE       &    0.41288 &  (0.37, 0.45] &        0.33580 &   (0.3, 0.38] &     (0.0, 0.15]            * &  0.0059 \\
gold\_response &    0.32384 &  (0.28, 0.36] &        0.46624 &  (0.43, 0.51] &  (-0.22, -0.07]            * &  <0.001 \\
\bottomrule
\end{tabular}
\caption{Which response is more \textit{interesting}?}
\end{minipage}
\end{table}

\subsection{Human Evaluation}

We closely follow the protocol used in \citet{zhang2019dialogpt}. Systems were paired and each response pair was presented to 3 judges in random order. Judges expressed their preference on a 3 point Likert scale. We use a majority vote for each response pair to decide whether a specific baseline, the pivot (DEMI), or neither, performed better. We then bootstrap the set of majority votes to obtain a 99\% confidence interval (CI) on the expected difference between the baseline and DEMI. If this confidence interval contains 0, the difference is deemed insignificant. We also compute p-values from the confidence intervals\footnote{\url{https://www.bmj.com/content/343/bmj.d2304}}.

In the following tables, the ``pivot'' is always the system given by DEMI.
Pairings where the pairwise confidence interval is marked with ``*'' have a significant difference.